%% file: main.tex
\documentclass{article}

\usepackage{tex/fancyhdr,tex/natbib}
\usepackage{tex/iclr2027_conference,times}
\iclrfinalcopy
\input{tex/math_commands.tex}

\usepackage{amsmath,amssymb,amsthm,mathtools}
\usepackage{booktabs,tabularx,multirow}
\usepackage{wrapfig}
\usepackage{placeins}
\usepackage{algorithm}
\usepackage{algpseudocode}
\usepackage{graphicx}
\usepackage{xcolor}
\usepackage{colortbl}
\usepackage{xspace}
\usepackage{microtype}
\usepackage{enumitem}
\usepackage{tikz}
\usepackage{tcolorbox}
\tcbuselibrary{breakable}
\usepackage{fontawesome}
\usetikzlibrary{arrows.meta,positioning,calc,fit,backgrounds}
\usepackage{pgfplots}
\pgfplotsset{compat=1.18}
\usepackage[bottom]{footmisc}
\usepackage{hyperref}
\usepackage{url}

\definecolor{ink}{HTML}{263238}
\definecolor{nativeblue}{HTML}{3568A8}
\definecolor{methodgreen}{HTML}{17856A}
\definecolor{riskred}{HTML}{C64B45}
\definecolor{gold}{HTML}{D39B2A}
\definecolor{methodrow}{HTML}{F2F2F2}
\hypersetup{colorlinks=true,citecolor=nativeblue,linkcolor=nativeblue,urlcolor=nativeblue}

\newcommand{\papertablefont}{\scriptsize}
\newcommand{\method}{\textup{\textsc{HeirCommit}}\xspace}
\newcommand{\system}{\textup{\textsc{EvoHeir}}\xspace}
\newcommand{\recovery}{\textup{\textsc{Experience Recovery}}\xspace}

\xspaceaddexceptions{]\%}
\setlist[itemize]{leftmargin=*,topsep=2pt,itemsep=1pt}
\setlist[enumerate]{leftmargin=*,topsep=2pt,itemsep=1pt}
\newtheorem{proposition}{Proposition}
\newtheorem{theorem}{Theorem}
\newtheorem{corollary}{Corollary}

\input{tex/results_macros.tex}

\renewcommand{\AblFullMath}{31.82}

\title{Selection Is Not Learning:\\Choosing Model Futures}

\author{Haoyi Yang\thanks{Equal contribution. Either author may list their name first on their CV.} \quad Kshitij Mishra\footnotemark[1] \quad Nils Lukas \quad Salem Lahlou \\
Mohamed bin Zayed University of Artificial Intelligence, Abu Dhabi, UAE \\
\texttt{\{haoyi.yang,kshitij.mishra,nils.lukas,salem.lahlou\}@mbzuai.ac.ae}}

\begin{document}
\maketitle

\begin{abstract}
\input{sections/abstract}
\end{abstract}

\input{sections/introduction}
\input{sections/related_work}
\input{sections/problem_setup}
\input{sections/method}
\input{sections/theory}
\input{sections/experiments}
\input{sections/limitations_conclusion}

\bibliography{references}
\bibliographystyle{tex/iclr2027_conference}

\clearpage
\appendix
\raggedbottom %
\numberwithin{equation}{section}
\section*{Appendix}
\input{appendix/experimental_setup}
\input{appendix/detailed_method}
\input{appendix/theoretical_analysis}
\FloatBarrier
\input{appendix/task_selection_inheritance_audits}
\input{appendix/extended_self_evolution_results}
\section{Disentangling parameter and experience inheritance}
\label{app:parameter-experience-inheritance}
\input{appendix/parameter_experience_inheritance}
\FloatBarrier
\input{appendix/compute_performance}
\input{appendix/prompt_templates}
\end{document}

%% file: tex/math_commands.tex
\usepackage{amsmath,amsfonts,bm}

\def\eqref#1{equation~\ref{#1}}

\def\1{\bm{1}}

\DeclareMathAlphabet{\mathsfit}{\encodingdefault}{\sfdefault}{m}{sl}
\SetMathAlphabet{\mathsfit}{bold}{\encodingdefault}{\sfdefault}{bx}{n}

\newcommand{\E}{\mathbb{E}}

%% file: tex/results_macros.tex
\newcommand{\AblFullMath}{31.82}
\newcommand{\AblFullMathSpread}{0.55}
\newcommand{\AblFullGeneral}{30.52}
\newcommand{\AblFullGeneralSpread}{0.35}
\newcommand{\AblLineageMath}{30.77}
\newcommand{\AblLineageMathSpread}{0.57}
\newcommand{\AblLineageGeneral}{29.99}
\newcommand{\AblLineageGeneralSpread}{0.35}
\newcommand{\AblRandomMath}{30.54}
\newcommand{\AblRandomMathSpread}{0.44}
\newcommand{\AblRandomGeneral}{29.87}
\newcommand{\AblRandomGeneralSpread}{0.43}
\newcommand{\AblCurrentMath}{30.85}
\newcommand{\AblCurrentMathSpread}{0.69}
\newcommand{\AblCurrentGeneral}{30.04}
\newcommand{\AblCurrentGeneralSpread}{0.35}
\newcommand{\AblNoConfirmMath}{29.80}
\newcommand{\AblNoConfirmMathSpread}{1.13}
\newcommand{\AblNoConfirmGeneral}{29.33}
\newcommand{\AblNoConfirmGeneralSpread}{0.65}

\newcommand{\AgentFalseFrontier}{52.2\xspace}
\newcommand{\RZeroFalseFrontier}{42.0\xspace}

\newcommand{\AuditAgentSelected}{8.60}

\newcommand{\AuditAgentControl}{0.20}

\newcommand{\AuditRZeroSelected}{4.10}

\newcommand{\AuditRZeroControl}{-0.10}

\newcommand{\AuditUpdateCount}{96}
\newcommand{\AuditNonpositiveCount}{60}

\newcommand{\AuditNonpositivePercent}{62.5}

\newcommand{\SkillHeirNativeCoverage}{100.00\%\xspace}
\newcommand{\SkillHeirNativeUseful}{100.00\%\xspace}
\newcommand{\SkillHeirNativeHarmful}{57.81\%\xspace}
\newcommand{\SkillHeirCoverage}{39.06\%\xspace}
\newcommand{\SkillHeirRandomUseful}{39.06\%\xspace}
\newcommand{\SkillHeirRandomHarmful}{57.81\%\xspace}
\newcommand{\SkillHeirUseful}{77.78\%\xspace}
\newcommand{\SkillHeirHarmful}{16.00\%\xspace}
\newcommand{\SkillHeirNativeGainMilli}{$-3.47$\xspace}
\newcommand{\SkillHeirRandomGainMilli}{$-1.36$\xspace}
\newcommand{\SkillHeirGainMilli}{$+2.49$\xspace}
\newcommand{\SkillHeirRandomVsNativeMilli}{$+2.11$\xspace}
\newcommand{\SkillHeirDeltaMilli}{$+5.96$\xspace}

\newcommand{\AgentHeirNativeCoverage}{100.00\%\xspace}
\newcommand{\AgentHeirNativeUseful}{100.00\%\xspace}
\newcommand{\AgentHeirNativeHarmful}{46.88\%\xspace}
\newcommand{\AgentHeirCoverage}{51.56\%\xspace}
\newcommand{\AgentHeirRandomUseful}{51.56\%\xspace}
\newcommand{\AgentHeirRandomHarmful}{46.88\%\xspace}
\newcommand{\AgentHeirUseful}{82.35\%\xspace}
\newcommand{\AgentHeirHarmful}{15.15\%\xspace}
\newcommand{\AgentHeirNativeGainMilli}{$+2.80$\xspace}
\newcommand{\AgentHeirRandomGainMilli}{$+1.44$\xspace}
\newcommand{\AgentHeirGainMilli}{$+14.35$\xspace}
\newcommand{\AgentHeirRandomVsNativeMilli}{$-1.36$\xspace}
\newcommand{\AgentHeirDeltaMilli}{$+11.55$\xspace}

\newcommand{\RZeroHeirNativeCoverage}{100.00\%\xspace}
\newcommand{\RZeroHeirNativeUseful}{100.00\%\xspace}
\newcommand{\RZeroHeirNativeHarmful}{76.56\%\xspace}
\newcommand{\RZeroHeirCoverage}{23.44\%\xspace}
\newcommand{\RZeroHeirRandomUseful}{23.44\%\xspace}
\newcommand{\RZeroHeirRandomHarmful}{76.56\%\xspace}
\newcommand{\RZeroHeirUseful}{86.67\%\xspace}
\newcommand{\RZeroHeirHarmful}{13.33\%\xspace}
\newcommand{\RZeroHeirNativeGainDisplay}{$-23.17$\xspace}
\newcommand{\RZeroHeirRandomGainDisplay}{$-5.43$\xspace}
\newcommand{\RZeroHeirGainDisplay}{$+10.19$\xspace}
\newcommand{\RZeroHeirRandomVsNativeDisplay}{$+17.74$\xspace}
\newcommand{\RZeroHeirDeltaDisplay}{$+33.36$\xspace}

%% file: sections/abstract.tex
Self-evolving language models
train on tasks they generate, so every round must decide which updated
model the next round inherits.
Current loops such as R-Zero, Agent0, and Skill Self-Play make this
decision through the task, admitting tasks that look learnable and inheriting
every resulting update.
We identify a \textbf{selection-to-update gap} in these loops, where a
task that looks learnable need not produce a useful update.
Most updates on verified tasks fail to make correct answers more likely,
yet these loops inherit them all.
\textbf{\method} freezes the
updated model and uses fresh evaluation questions to decide whether to inherit
it or retain the current model. \textbf{\system} separates parameter inheritance from
experience inheritance, using current and archived experience to repair
failures and preserve successes before final confirmation. Our analysis
characterizes selection-induced optimism and decomposes regret into selection,
absorption, and confirmation errors.
\method{} cuts
the harmful share of accepted updates by
68--83\%
relative to random inheritance at the same acceptance rate.
Across four backbones, \system{} achieves the best mathematics and
general-reasoning averages, raising the math average by
5.3\%
relative to the strongest baseline, Agent0.
The framework lets useful experience
persist even when the model that produced it is not inherited.

%% file: sections/introduction.tex
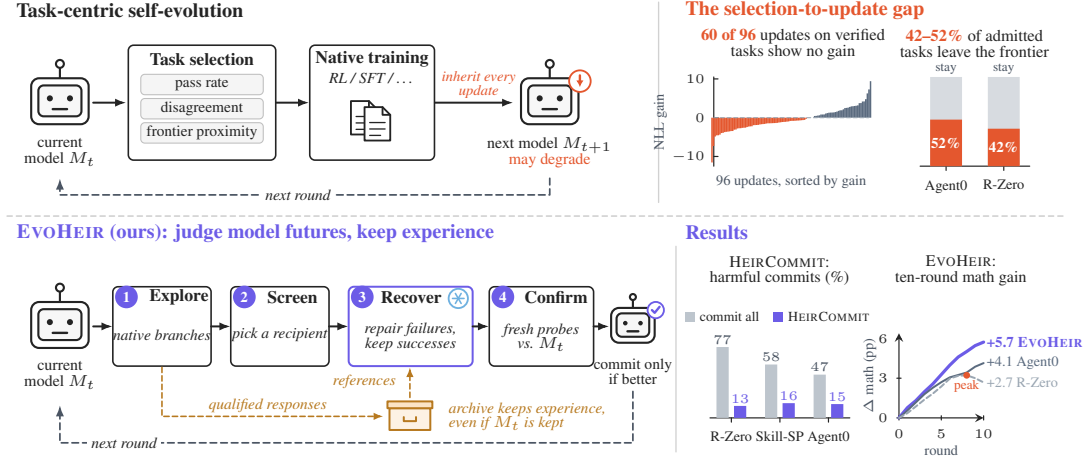
\begin{figure}[!b]
  \vspace{-8pt}
  \centering
  \input{figures/teaser-tikz}
  \vspace{-4pt}
  \caption[Selecting a learnable task does not guarantee learning]{\textbf{Selecting a learnable task does not guarantee learning, so \system judges model futures and keeps experience.}
  Top: task-centric loops inherit every update, yet 60 of 96 updates on verified tasks show no gain, and 42--52\% of admitted tasks leave the frontier under fresh responses.
  Bottom: \system explores, screens a recipient, repairs its failures with current and archived references while keeping its successes, and commits the result only if it beats the current model on fresh probes. The archive keeps experience either way.
  \method cuts the harmful share of commits on every host, and over ten rounds \system keeps climbing to $+5.7$ math points while R-Zero peaks at round 8 and slips (means over four backbones). All plots show measured data.}
  \label{fig:motivation}
\end{figure}

\section{Introduction}

Self-evolving language models generate tasks and train on their own outputs
\citep{huang2026rzeroselfevolvingreasoningllm,xia2026agent0unleashingselfevolvingagents,huang2026skillselfplaypushingfrontier}. Each inherited update
changes both the learner and the curriculum of later rounds, allowing progress
or regressions to accumulate~\citep{shao2026agentmisevolveemergentrisks}.
Self-training can even
raise a model's target metric and then collapse it within the same training
campaign~\citep{lin2026selfimprovementselfregress}.
The central decision is
therefore \emph{which learning outcomes to carry forward}.

Task-side selection leaves a \textbf{selection-to-update gap}: pass rate,
disagreement, and frontier proximity describe pre-training behavior, not the
transfer produced by an optimizer update. In our audits, 42--52\% of native
Agent0/R-Zero admissions leave the frontier under independent responses.
Even with exact-valid tasks, \AuditNonpositiveCount{} of \AuditUpdateCount{}
Skill Self-Play (Skill-SP) updates have nonpositive correct-target NLL gain.
Learning opportunities must ultimately be judged by the states they produce.

\textbf{\method} judges realized transitions: it freezes a successor,
compares it with the current model (the root) on fresh probes, and inherits
that state or retains root. Fresh evidence separates confirmation from the
noise that favored the proposal during selection.

The deeper question is the \emph{granularity of inheritance}. SOAR uses actual
student improvement to reward a teacher~\citep{sundaram2026teachingmodelsteachthemselves}, and APE
evaluates candidate fine-tunings before accepting updates~\citep{marín2025apeselectivefinetuningacceptance}.
Pre-commit skill verification~\citep{shang2026selfevolutionbackfiresprecommitgating}
and anytime-valid certificates~\citep{sengupta2026selfevolvingagentsanytimevalid}
gate agent updates while the base model stays frozen.
Whole-state acceptance nevertheless couples a branch's gains and losses.
\emph{Must a useful experience donor also supply the next learner's parameters?}
Separating these choices opens a learning path beyond whole-state acceptance
(Figure~\ref{fig:motivation}).

\textbf{\system} separates parameter inheritance from experience inheritance.
Screening picks the recipient that supplies the next learner's parameters,
either a branch endpoint or the current model. \textbf{\recovery} then
repairs this recipient's failures with current and archived reference
solutions while preserving its successes.

\method confirms \emph{the completed post-absorption state} against root.
Qualified experience enters the archive regardless of this decision:
parameter inheritance can stop without discarding material that helps a later
recipient.

We evaluate \method{} on the native updates of three self-evolution
hosts, and \system on four backbones over ten rounds.
\method cuts the harmful share of accepted updates by 68--83\%
relative to random inheritance at the same acceptance rate, on every host.
\system achieves the highest mathematics and general-reasoning averages on all
four backbones and beats Agent0 on 45 of 48 benchmark--backbone pairs.
References from eight of ten donors weaker than the current model still
improve the recipient.

Our contributions are:
\begin{enumerate}\setlength{\itemsep}{0pt}
\item We identify the \textbf{selection-to-update gap} and
prove that pre-update behavior alone cannot identify an update's
value.
Reusing screening evidence to accept a branch pays a union bound over branches,
which fresh confirmation removes.
\item We introduce \textbf{\method}, a source-independent rule that commits a
frozen successor only when it beats the current model on fresh probes. On the
native updates of R-Zero, Agent0, and Skill-SP, it turns negative pool gain
positive on two of the three hosts.
\item We develop \textbf{\system}, which decides separately which parameters the
next learner inherits and which experience it absorbs. Across four backbones it
raises the base model's math average by 7.9--30.1\%, the most of
any compared method.
\item We show that \textbf{experience outlives the model that produced it}.
Recovery beats direct fine-tuning on the same 64 references by 1.93 pp on
average, and experience from rounds whose model was rejected adds 0.66 pp
beyond committed-round history.
\end{enumerate}

%% file: figures/teaser-tikz.tex
\input{figures/eh-palette}
\input{figures/teaser-icons}
\input{figures/teaser-plots}
\begin{tikzpicture}[x=1cm,y=1cm,text=ehInk,line cap=round,line join=round,
  font=\rmfamily\fontsize{6.4}{7.4}\selectfont,
  flow/.style={-{Latex[length=1.5mm,width=1.1mm]},draw=ehInk,line width=.75pt},
  nextround/.style={-{Latex[length=1.5mm,width=1.1mm]},draw=ehInk2,line width=.65pt,dash pattern=on 2.5pt off 1.5pt},
  box/.style={draw=ehInk,line width=.6pt,rounded corners=2pt,fill=white},
  pill/.style={draw=black!25,fill=black!4,rounded corners=1.5pt,line width=.4pt,minimum width=1.62cm,
               minimum height=.27cm,inner sep=0pt,font=\rmfamily\fontsize{5.8}{6.4}\selectfont},
  head/.style={font=\rmfamily\bfseries\fontsize{6.6}{7.6}\selectfont},
  row/.style={font=\rmfamily\bfseries\fontsize{7.6}{8.6}\selectfont,anchor=west},
  note/.style={font=\rmfamily\itshape\fontsize{5.8}{6.6}\selectfont,align=center},
  stepno/.style={circle,fill=ehAccent,text=white,inner sep=0pt,minimum size=.30cm,
               font=\rmfamily\bfseries\fontsize{5.6}{6}\selectfont}]
\path[use as bounding box] (0,-.05) rectangle (13.97,6.02);

\node[row] at (0,5.84) {Task-centric self-evolution};
\node[row,text=ehHarm] at (8.85,5.84) {The selection-to-update gap};
\ehAgent{.70}{4.62}{.95}{0}
\node[note,font=\rmfamily\fontsize{6}{7}\selectfont] at (.70,3.97) {current\\[-1pt]model $M_t$};
\draw[flow] (1.14,4.62)--(1.58,4.62);
\node[box,minimum width=1.95cm,minimum height=1.52cm] (sel) at (2.58,4.62) {};
\node[head] at (2.58,5.20) {Task selection};
\node[pill] at (2.58,4.86) {pass rate};
\node[pill] at (2.58,4.54) {disagreement};
\node[pill] at (2.58,4.22) {frontier proximity};
\draw[flow] (3.56,4.62)--(4.00,4.62);
\node[box,minimum width=1.65cm,minimum height=1.52cm] (nat) at (4.83,4.62) {};
\node[head] at (4.83,5.20) {Native training};
\node[note] at (4.83,4.95) {RL / SFT / \dots};
\ehPage{4.72}{4.42}{.85}\ehPage{4.90}{4.34}{.85}
\draw[flow] (5.66,4.62)--(6.70,4.62);
\node[note,text=ehHarm,font=\rmfamily\itshape\fontsize{5.5}{6.3}\selectfont] at (6.24,4.86) {inherit every\\[-1pt]update};
\ehAgent{7.18}{4.62}{.95}{1}
\node[note,font=\rmfamily\fontsize{6}{7}\selectfont] at (7.18,3.97) {next model $M_{t+1}$\\[-1pt]\textcolor{ehHarm}{may degrade}};
\draw[nextround] (7.18,3.62)--(7.18,3.40)--(.70,3.40)--(.70,3.62);
\node[fill=white,inner sep=1.5pt,note] at (3.85,3.40) {next round};
\draw[black!18,line width=.5pt] (8.62,3.30)--(8.62,5.95);
\node[font=\rmfamily\fontsize{6.2}{7}\selectfont,align=center] at (10.40,5.40) {\textcolor{ehHarm}{\textbf{60 of 96}} updates on verified\\[-1pt]tasks show no gain};
\ehPlotNoGain{9.35cm,3.78cm}
\node[font=\rmfamily\fontsize{6.2}{7}\selectfont,align=center] at (12.80,5.40) {\textcolor{ehHarm}{\textbf{42--52\%}} of admitted\\[-1pt]tasks leave the frontier};
\ehPlotFrontier{12.08cm,3.78cm}

\draw[black!22,line width=.5pt,dash pattern=on 2pt off 1.5pt] (0,3.12)--(13.97,3.12);

\node[row,text=ehAccent] at (0,2.90) {\system\ (ours): judge model futures, keep experience};
\node[row,text=ehAccent] at (8.85,2.90) {Results};
\ehAgent{.70}{1.66}{.95}{0}
\node[note,font=\rmfamily\fontsize{6}{7}\selectfont] at (.70,1.02) {current\\[-1pt]model $M_t$};
\draw[flow] (1.14,1.66)--(1.40,1.66);
\foreach \x/\n/\t/\s in {2.05/1/Explore/native branches, 3.61/2/Screen/pick a recipient} {
  \node[box,minimum width=1.30cm,minimum height=1.05cm] at (\x,1.66) {};
  \node[stepno] at (\x-.47,2.05) {\n};
  \node[head,anchor=west] at (\x-.29,2.05) {\t};
  \node[note] at (\x,1.55) {\s};}
\draw[flow] (2.70,1.66)--(2.96,1.66);
\draw[flow] (4.26,1.66)--(4.52,1.66);
\node[box,minimum width=1.64cm,minimum height=1.05cm,draw=ehAccent,line width=.75pt] at (5.34,1.66) {};
\node[stepno] at (4.70,2.05) {3};
\node[head,anchor=west] at (4.82,2.05) {Recover};
\node[note] at (5.34,1.52) {repair failures,\\[-1pt]keep successes};
\ehBadgeSnow{5.97}{2.03}{.85}
\draw[flow] (6.16,1.66)--(6.38,1.66);
\node[box,minimum width=1.38cm,minimum height=1.05cm] at (7.07,1.66) {};
\node[stepno] at (6.56,2.05) {4};
\node[head,anchor=west] at (6.72,2.05) {Confirm};
\node[note] at (7.07,1.52) {fresh probes\\[-1pt]vs.\ $M_t$};
\draw[flow] (7.76,1.66)--(7.97,1.66);
\ehAgent{8.28}{1.66}{.72}{3}
\node[note,font=\rmfamily\fontsize{5.8}{6.6}\selectfont,align=center] at (8.28,1.07) {commit only\\[-1pt]if better};
\ehArchive{5.33}{.47}{.90}
\node[note,text=ehArchiveInk,anchor=west,align=left] at (5.72,.47) {archive keeps experience,\\[-1pt]even if $M_t$ is kept};
\draw[-{Latex[length=1.3mm,width=1mm]},draw=ehArchiveInk,line width=.6pt,dash pattern=on 2pt off 1.2pt] (5.33,.72)--(5.33,1.12);
\draw[-{Latex[length=1.3mm,width=1mm]},draw=ehArchiveInk,line width=.6pt,dash pattern=on 2pt off 1.2pt] (2.05,1.12)--(2.05,.47)--(4.98,.47);
\node[note,text=ehArchiveInk] at (3.45,.62) {qualified responses};
\node[note,text=ehArchiveInk,anchor=east] at (5.27,.93) {references};
\draw[nextround] (8.28,.80)--(8.28,.12)--(.70,.12)--(.70,.62);
\node[fill=white,inner sep=1.5pt,note] at (1.55,.12) {next round};
\draw[black!18,line width=.5pt] (8.85,.00)--(8.85,2.72);
\node[font=\rmfamily\fontsize{6.1}{6.9}\selectfont,align=center] at (10.22,2.40) {\method:\\[-1pt]harmful commits (\%)};
\ehPlotHarm{9.30cm,.45cm}
\node[font=\rmfamily\fontsize{6.1}{6.9}\selectfont,align=center] at (12.62,2.40) {\system:\\[-1pt]ten-round math gain};
\ehPlotTraj{11.80cm,.45cm}
\end{tikzpicture}

%% file: sections/related_work.tex
\section{Related Work}
\label{sec:related}

\paragraph{Self-evolution and update acceptance.}
STaR, Self-Instruct, and SPIN learn from self-generated supervision
\citep{zelikman2022star,wang2023self,chen2024self}; \citet{alami2024investigating} regularize SPIN with a KL term and fictitious play across iterations.
Task generation uses self-play in Absolute Zero and SPELL
\citep{zhao2025absolutezeroreinforcedselfplay,yang2026spellselfplayreinforcementlearning} and student feedback in
DataEnvGym~\citep{khan2025dataenvgymdatagenerationagents}. R-Zero, Agent0, and Skill-SP adapt
curricula~\citep{huang2026rzeroselfevolvingreasoningllm,xia2026agent0unleashingselfevolvingagents,huang2026skillselfplaypushingfrontier}; SOAR and
Actor-Curator use realized improvement~\citep{sundaram2026teachingmodelsteachthemselves,gu2026actorcuratorcoadaptivecurriculumlearning}. Earlier automatic curricula select tasks by measured learning progress~\citep{graves2017automated,matiisen2020teacher}, a signal that \citet{willems2020mastering} argue is flawed, proposing mastering rates instead.
At the agent level, harness updates need not bring benefit either~\citep{lin2026harnessupdatingnotharnessbenefit}.
Update acceptance appears in AlphaGo Zero, but not AlphaZero
\citep{silver2017mastering,silver2017masteringchess}, candidate fine-tuning in
APE~\citep{marín2025apeselectivefinetuningacceptance}, and nonparametric agent evolution in
PACE~\citep{shawn2026paceanytimevalidacceptancetests}.
Agent-side gates verify distilled skills before commit or certify adapter and harness updates while the base model stays frozen~\citep{shang2026selfevolutionbackfiresprecommitgating,sengupta2026selfevolvingagentsanytimevalid}.
\system separately chooses the recipient and
experience donors, then confirms the post-absorption successor.

\paragraph{Privileged-context distillation.}
GKD trains on student-generated prefixes~\citep{ICLR2024_5be69a58}.
SDFT, SDPO, and OPSD condition a self-teacher on
demonstrations, feedback, or privileged context
\citep{shenfeld2026selfdistillationenablescontinuallearning,hübotter2026reinforcementlearningselfdistillation,zhao2026self},
while OPCD internalizes contextual knowledge~\citep{ye2026onpolicycontextdistillationlanguage}.
Recovery combines recipient-specific repair and preservation with cross-branch
references, fixed recipient trajectories, and post-absorption confirmation.

\paragraph{Experience persistence.}
Replay retains past experience~\citep{rolnick2019experience}; genetic distillation
combines policies~\citep{gangwani2018policyoptimizationgeneticdistillation}. Online experiential learning and
SPEE internalize accumulated experience~\citep{ye2026onlineexperientiallearninglanguage,ren2026selfimprovinglargelanguagemodels}.
\system reassesses experience through the current recipient and retains it
independently of parameter acceptance: a rejected state can remain a useful donor.

%% file: sections/problem_setup.tex
\section{Self-Evolution Environment}
\label{sec:setup}

\begin{wrapfigure}[14]{r}{0.42\textwidth}
  \vspace{-12pt}
  \centering
  \includegraphics[width=0.42\textwidth]{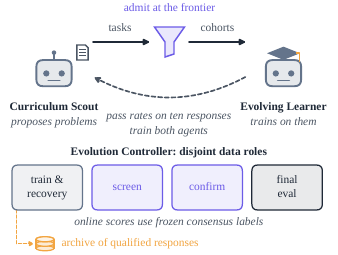}
  \caption{\textbf{Self-evolution environment.} Two agents co-evolve through the
  Evolution Controller.}
  \label{fig:environment}
  \vspace{-10pt}
\end{wrapfigure}
Self-evolution runs in rounds, and each round ends with an inheritance decision.
The Curriculum Scout proposes problems, the Evolution Controller admits those at
the Evolving Learner's frontier, and native training updates both agents in an
R-Zero-style loop~\citep{huang2026rzeroselfevolvingreasoningllm}, as
Figure~\ref{fig:environment} shows.
A round starts from the learner's \emph{root state} $S_t=(\theta_t,\omega_t)$,
its parameters and resumable training state, and trains two native branches
from it.
Admission uses pre-update behavioral observations $O_t$, such as the pass rate
over ten learner responses, and the branch endpoints form
$\mathcal C_t^{\rm nat}$.
Qualified responses from every branch enter a bounded archive $M_t$.
Training and Recovery, screening, confirmation and final evaluation use disjoint
question pools, and online scores use frozen consensus labels, as
Appendix~\ref{app:consensus-protocol} details.

\paragraph{Value and inheritance.}
For a declared task distribution $P_t$ and utility $q$, the value of a state is
its expected gain over the root,
\begin{equation}
 V_t(S)=\E_{z\sim P_t}[q(S;z)-q(S_t;z)], \qquad V_t(S_t)=0.
 \label{eq:value}
\end{equation}
Capability uses exact correctness or sampled pass@1, and the update audits use
the reduction in correct-target negative log-likelihood (NLL), which is not an
accuracy difference.
For binary correctness, value splits into acquired and lost successes,
\begin{equation}
 V_t(S)=\Pr(q(S)=1,q(S_t)=0)-\Pr(q(S)=0,q(S_t)=1).
 \label{eq:gain-loss}
\end{equation}
The \emph{selection-to-update gap} is the failure of task-side admission to
identify positive value.
Screening picks a recipient $S_t^{\rm sel}$ from
$\mathcal C_t^{\rm nat}\cup\{S_t\}$, Recovery turns it into a proposal $S_t^+$,
and confirmation keeps $S_t$ or adopts $S_t^+$ as $S_{t+1}$.
Transition regret compares the inherited state with the best state constructed,
\begin{equation}
 R_t=\max_{S\in\mathcal C_t^{\rm nat}\cup\{S_t,S_t^+\}}V_t(S)-V_t(S_{t+1}),
 \label{eq:regret}
\end{equation}
and absorption gain $G_t^{\rm abs}=V_t(S_t^+)-V_t(S_t^{\rm sel})$ measures the
change from recipient to proposal.

\paragraph{Hosts, backbones, and benchmarks.}
\method is audited on the native updates of R-Zero, Agent0 and Skill-SP,
which propose tasks through a challenger, a tool-using curriculum agent and a
skill controller,
respectively~\citep{huang2026rzeroselfevolvingreasoningllm,xia2026agent0unleashingselfevolvingagents,huang2026skillselfplaypushingfrontier}.
The self-evolution study runs every method for ten rounds with three training
seeds from \texttt{Qwen3-4B-Base} and
\texttt{Qwen3-8B-Base}~\citep{yang2025qwen3technicalreport},
\texttt{OLMo-3-7B}~\citep{olmo2026olmo3}, and the instruction-tuned
\texttt{Gemma-4-E4B-it}~\citep{gemmateam2026gemma4technicalreport}.
It evaluates seven mathematics and five general-reasoning
benchmarks~\citep{hendrycks2021measuringmathematicalproblemsolving,lewkowycz2022solvingquantitativereasoningproblems,he2024olympiadbenchchallengingbenchmarkpromoting,ballon2026benchmarkssaturatemodelgets,mayilvahanan2025mathbeyondbenchmarkrlexpand,pteam2025supergpqascalingllmevaluation,wang2024mmluprorobustchallengingmultitask,kazemi2025bigbenchextrahard,white2025livebenchchallengingcontaminationlimitedllm,ma2025korbenchbenchmarkinglanguagemodels},
listed with their metrics in Appendix~\ref{app:extended-results}.

%% file: sections/method.tex
\section{\system: Exploration, Experience Recovery, and \method}
\label{sec:method}
\begin{figure}[t]
  \centering
  \input{figures/method-tikz}
  \caption[One \system round]{\textbf{One \system round.} Steps follow Algorithm~\ref{alg:evoheir}, and the archive keeps qualified experience whether or not the round commits.}
  \label{fig:method}
\end{figure}
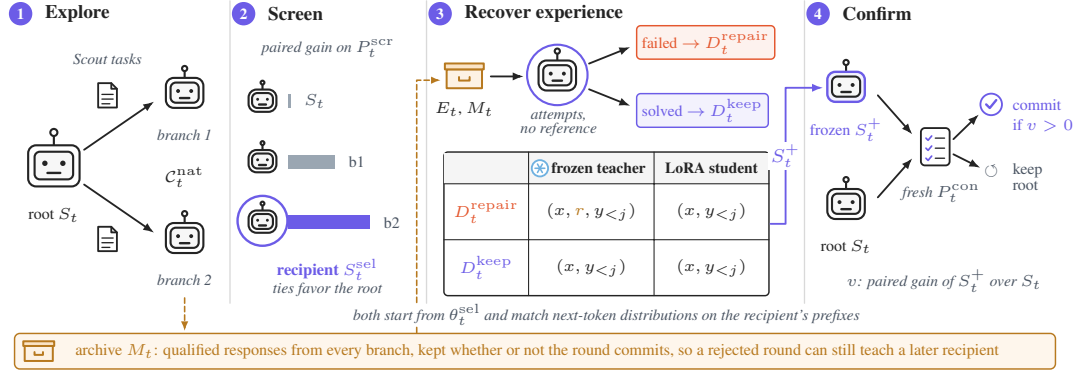

\system gives exploration two routes into the next learner, selected
parameters and absorbed experience. Each round runs the four steps of
Figure~\ref{fig:method}. It (1) explores native branches from the root,
(2) screens them with the root for a parameter recipient, (3) recovers current
and archived experience into that recipient, and (4) confirms the result against
the root on fresh probes.

\subsection{Explore and Screen a Recipient}

On the base loop of Section~\ref{sec:setup}, \system builds
multi-branch exploration, screening of realized updates, experience absorption,
and independent confirmation. Native endpoints and
any predeclared training-prefix checkpoints form $\mathcal C_t^{\rm nat}$.
Each branch supplies a resumable state and responses that qualify as experience. An \emph{Experience Donor} contributes the latter without
having to supply inherited parameters or remain available as a teacher model.

The controller freezes native candidates before shared screening on
$P_t^{\rm scr}$. Paired root-relative scores select the recipient
$S_t^{\rm sel}$ from the candidates and root, with ties going to the root and then to a fixed branch order. Screening chooses the parameter starting point for absorption.

\subsection{Recover Experience without Inheriting the Donor}

\paragraph{Experience across branches and rounds.}
Current exploration records $E_t$ and a bounded archive $M_t$ supply candidate
problems and references with their generating checkpoint and verification
trace. Qualification checks completeness, answer validity, and problem-level
de-duplication. Recovery material is disjoint from screening, confirmation,
and final evaluation. Appendix~\ref{app:archive-inheritance} specifies archive
admission and retention.

\paragraph{Let the recipient identify repair and preservation.}
Recovery obtains reference-free recipient attempts by exact-checkpoint cache
reuse or one rollout batch. Under the same answer-equivalence rule, a failed attempt with a qualified reference enters $D_t^{\rm repair}$, and a successful attempt enters $D_t^{\rm keep}$. The groups are disjoint and specific to the recipient, so an archived success becomes repair material when the current recipient cannot reproduce it. A predeclared budget bounds selection from both groups.

\paragraph{A frozen recipient teaches a recipient-initialized student.}

The frozen teacher and student in Figure~\ref{fig:method} are two copies of
the selected Evolving Learner, both initialized from $\theta_t^{\rm sel}$.
For a recipient-generated trajectory $y$ on problem $x$, all three
distributions below predict the next token at the same prefix $y_{<j}$.
The repair teacher also sees a donor's reference $r$,
$T^{\rm repair}_j=\operatorname{sg}[p_{\theta_t^{\rm sel}}(\cdot\mid x,r,y_{<j})]$,
while the keep teacher and the student see only the problem,
$T^{\rm keep}_j=\operatorname{sg}[p_{\theta_t^{\rm sel}}(\cdot\mid x,y_{<j})]$ and
$P_{\theta,j}=p_\theta(\cdot\mid x,y_{<j})$.
The student matches the repair teacher on failed attempts and the keep teacher
on successful ones,
\begin{equation}
 \mathcal L_{\rm rec}(\theta)=
 \left\langle d(T^{\rm repair}_j,P_{\theta,j})\right\rangle_{D_t^{\rm repair}}
 +\lambda_{\rm keep}
 \left\langle d(T^{\rm keep}_j,P_{\theta,j})\right\rangle_{D_t^{\rm keep}},
 \label{eq:recovery-loss}
\end{equation}
where $d$ is the forward KL divergence over the full vocabulary and
$\langle\cdot\rangle_D$ averages response-token terms per question over group
$D$, normalized by the repair-plus-keep batch size.
We use $\lambda_{\rm keep}=1$, LoRA rank 8 and at most 16 updates, as
Appendix~\ref{app:implementation-configuration} lists.
The recipient trajectories and groups stay fixed during absorption, which
produces $S_t^+$, and an empty experience pool leaves the recipient unchanged.
Only the resulting learner is deployed, since teacher conditioning is a
training operation.

\subsection{\method: Confirm the Final Proposal}

\method evaluates the state that would actually persist, including the
changes produced by absorption. The controller freezes the final proposal
before scoring the held-out confirmation pool $P_t^{\rm con}$. Its paired
root-relative value determines inheritance:
\begin{equation}
 S_{t+1}=
 \begin{cases}
 S_t^+,&\widehat V^{\rm con}_{t}(S_t^+)>\tau_t,\\
 S_t,&\text{otherwise}.
 \end{cases}
 \label{eq:heircommit}
\end{equation}
The reported runs use $\tau_t=0$, and a positive confidence margin gives the certified mode of Theorem~\ref{thm:commit}. Rejection retains root
without falling back to the recipient or testing another proposal. A proposal
identical to root is a zero-valued no-op. Later rounds receive fresh evidence.

\method{} does not depend on how candidates are built, so an external host can apply it to a single native update, and \system{} supplies construction and search.
Accepted absorbed weights are merged into the next native root, with a fresh
native optimizer and scheduler. Rejection restores the complete root state, including the scheduler and random generators. Qualified
experience is archived independently of this decision. Algorithm~\ref{alg:evoheir}
and the state-management details appear in Appendix~\ref{app:detailed-method}.

%% file: figures/method-tikz.tex
\input{figures/eh-palette}
\input{figures/teaser-icons}
\begin{tikzpicture}[x=1cm,y=1cm,text=ehInk,line cap=round,line join=round,
  font=\rmfamily\fontsize{6.4}{7.4}\selectfont,
  flow/.style={-{Latex[length=1.5mm,width=1.1mm]},draw=ehInk,line width=.7pt},
  flowacc/.style={flow,draw=ehAccent},
  arch/.style={-{Latex[length=1.3mm,width=1mm]},draw=ehArchiveInk,line width=.6pt,dash pattern=on 2pt off 1.2pt},
  stepno/.style={circle,fill=ehAccent,text=white,inner sep=0pt,minimum size=.31cm,
                 font=\rmfamily\bfseries\fontsize{5.8}{6}\selectfont},
  stage/.style={font=\rmfamily\bfseries\fontsize{7.2}{8}\selectfont,anchor=west},
  note/.style={font=\rmfamily\itshape\fontsize{5.8}{6.6}\selectfont,align=center,text=ehInk2},
  lab/.style={font=\rmfamily\fontsize{6}{7}\selectfont,align=center},
  chip/.style={rounded corners=1.5pt,line width=.55pt,minimum height=.30cm,inner sep=2pt,
               font=\rmfamily\fontsize{5.9}{6.5}\selectfont}]
\path[use as bounding box] (0,0) rectangle (13.97,5.05);
\foreach \x in {2.95,5.55,10.55} {\draw[black!15,line width=.5pt] (\x,.95)--(\x,4.95);}

\node[stepno] at (.18,4.78) {1}; \node[stage] at (.38,4.78) {Explore};
\ehAgent{.62}{2.75}{.85}{0}
\node[lab] at (.62,2.12) {root $S_t$};
\ehPage{1.33}{3.70}{.62}\ehPage{1.33}{1.80}{.62}
\node[note] at (1.33,4.18) {Scout tasks};
\draw[flow] (1.02,2.98)--(1.95,3.62);
\draw[flow] (1.02,2.52)--(1.95,1.88);
\ehAgent{2.35}{3.72}{.62}{0}\node[note] at (2.35,3.18) {branch 1};
\ehAgent{2.35}{1.80}{.62}{0}\node[note] at (2.35,1.26) {branch 2};
\node[lab] at (2.35,2.62) {$\mathcal C_t^{\rm nat}$};

\node[stepno] at (3.13,4.78) {2}; \node[stage] at (3.33,4.78) {Screen};
\node[note] at (4.25,4.32) {paired gain on $P_t^{\rm scr}$};
\foreach \y/\n/\w/\c in {3.62/$S_t$/.03/ehGrid, 2.82/b1/.62/ehGrid, 2.02/b2/1.08/ehAccent} {
  \ehAgent{3.38}{\y}{.46}{0}
  \fill[\c] (3.72,\y-.09) rectangle (3.72+\w,\y+.09);
  \node[lab,anchor=west,text=ehInk2] at (3.78+\w,\y) {\n};}
\draw[ehAccent,line width=.9pt] (3.38,2.04) circle (.33);
\node[lab,text=ehAccent,font=\rmfamily\bfseries\fontsize{6.2}{7}\selectfont] at (4.24,1.40) {recipient $S_t^{\rm sel}$};
\node[note] at (4.24,1.13) {ties favor the root};

\node[stepno] at (5.73,4.78) {3}; \node[stage] at (5.93,4.78) {Recover experience};
\ehArchive{6.05}{3.95}{.78}
\node[lab] at (6.05,3.52) {$E_t$, $M_t$};
\draw[flow] (6.40,3.95)--(6.84,3.95);
\ehAgent{7.28}{3.95}{.62}{0}\draw[ehAccent,line width=.8pt] (7.28,3.97) circle (.40);
\node[note] at (7.30,3.33) {attempts,\\[-1pt]no reference};
\draw[flow] (7.70,4.12)--(8.25,4.34);
\draw[flow] (7.70,3.78)--(8.25,3.56);
\node[chip,draw=ehHarm,fill=ehHarm!8,text=ehHarm,anchor=west] at (8.32,4.40) {failed $\rightarrow$ $D_t^{\rm repair}$};
\node[chip,draw=ehAccent,fill=ehAccent!8,text=ehAccent,anchor=west] at (8.32,3.50) {solved $\rightarrow$ $D_t^{\rm keep}$};
\begin{scope}[shift={(5.78,1.05)}]
  \draw[ehInk,line width=.55pt,rounded corners=1.5pt] (0,0) rectangle (4.30,1.90);
  \fill[black!4] (.01,1.48) rectangle (4.29,1.89);
  \draw[ehInk,line width=.4pt] (0,1.48)--(4.30,1.48) (0,.74)--(4.30,.74) (1.12,0)--(1.12,1.90) (2.78,0)--(2.78,1.90);
  \ehBadgeSnow{1.27}{1.685}{.72}
  \node[lab,font=\rmfamily\bfseries\fontsize{5.8}{6.6}\selectfont,anchor=west,inner sep=1pt] at (1.38,1.685) {frozen teacher};
  \node[lab,font=\rmfamily\bfseries\fontsize{5.8}{6.6}\selectfont] at (3.54,1.685) {LoRA student};
  \node[lab,text=ehHarm] at (.56,1.11) {$D_t^{\rm repair}$};
  \node[lab,text=ehAccent] at (.56,.37) {$D_t^{\rm keep}$};
  \node[lab] at (1.95,1.11) {$(x,\textcolor{ehArchiveInk}{r},y_{<j})$};
  \node[lab] at (1.95,.37) {$(x,y_{<j})$};
  \node[lab] at (3.54,1.11) {$(x,y_{<j})$};
  \node[lab] at (3.54,.37) {$(x,y_{<j})$};
\end{scope}
\node[note] at (7.93,.80) {both start from $\theta_t^{\rm sel}$ and match next-token distributions on the recipient's prefixes};
\draw[flowacc] (10.10,2.00)--(10.30,2.00)--(10.30,3.80)--(10.80,3.80);
\node[lab,text=ehAccent,font=\rmfamily\bfseries\fontsize{6.6}{7.4}\selectfont,fill=white,inner sep=1pt] at (10.30,2.90) {$S_t^+$};

\node[stepno] at (10.73,4.78) {4}; \node[stage] at (10.93,4.78) {Confirm};
\ehAgent{11.10}{3.80}{.62}{4}\node[lab,text=ehAccent] at (11.10,3.25) {frozen $S_t^+$};
\ehAgent{11.10}{2.25}{.62}{0}\node[lab] at (11.10,1.70) {root $S_t$};
\draw[flow] (11.52,3.70)--(11.98,3.18);
\draw[flow] (11.52,2.35)--(11.98,2.87);
\ehProbes{12.28}{3.02}{.95}
\node[note] at (12.36,2.40) {fresh $P_t^{\rm con}$};
\draw[flow] (12.52,3.12)--(12.86,3.45);
\draw[flow] (12.52,2.92)--(12.88,2.74);
\ehBadgeCheck{13.02}{3.58}{.95}
\node[lab,text=ehAccent,anchor=west] at (13.18,3.58) {commit};
\node[lab,text=ehAccent,anchor=west] at (13.18,3.30) {if $v>0$};
\node[lab,text=ehInk2] at (13.02,2.70) {$\circlearrowleft$};
\node[lab,text=ehInk2,anchor=west,align=left] at (13.18,2.62) {keep\\[-1pt]root};
\node[note] at (12.40,1.25) {$v$: paired gain of $S_t^+$ over $S_t$};

\draw[ehArchiveInk,fill=ehArchive!10,line width=.55pt,rounded corners=2pt] (.10,.08) rectangle (13.87,.55);
\ehArchive{.42}{.31}{.62}
\node[note,text=ehArchiveInk,anchor=west,font=\rmfamily\fontsize{6.1}{7}\selectfont] at (.78,.31)
  {archive $M_t$: qualified responses from every branch, kept whether or not the round commits, so a rejected round can still teach a later recipient};
\draw[arch] (2.35,1.02)--(2.35,.58);
\draw[arch,-] (5.42,.56)--(5.42,.63);
\draw[arch] (5.42,.96)--(5.42,3.90)--(5.74,3.90);
\end{tikzpicture}

%% file: sections/theory.tex
\section{What Can Be Identified about a Model Future?}
\label{sec:theory}

Within a round, write $V_k$ for the value of native candidate
$k$, with $K\geq1$, and $V_0=0$ for root. Screening selects
$k^*=\arg\max_{0\leq k\leq K}\widehat V_k^{\rm scr}$, with
$\widehat V_0^{\rm scr}=0$. Write $V^+=V_t(S_t^+)$ for the final proposal's value.
Proofs and the adaptive-round extension appear in
Appendix~\ref{app:theoretical-analysis}.

\begin{proposition}[Task-side non-identifiability]
\label{prop:identifiability}
For loss-reduction utility $q=-\ell$ and any $\Delta>0$, every persistence
rule based only on pre-update behavioral observations $O_t$ faces two
differentiable learning worlds with identical $O_t$, equally valid targets,
and the same native step size, but values $+\Delta$ and $-\Delta$.
Against the better of that update and root, its worst-case expected regret
is at least $\Delta/2$.
\end{proposition}
The observation interface excludes update gradients. Identical root behavior
can therefore lead to opposite update effects, motivating evaluation of the
realized successor rather than task quality alone.

Screening also inflates the winner's score, $\E[\widehat V^{\rm scr}_{k^*}-V_{k^*}]\geq0$,
even when each candidate's score is unbiased~\citep{cawley2010over,russo2016controlling}, as
Proposition~\ref{prop:selected-optimism} in Appendix~\ref{app:optimism-confirmation} states.

Theorem~\ref{thm:commit} explains why \system screens and confirms on disjoint pools,
because reusing the screen to accept a branch pays a union bound over the $K$ branches, while a fresh pool does not, even when absorption follows recipient selection.

\begin{theorem}[Search--inheritance separation]
\label{thm:commit}
Conditional on the pre-screening history and frozen native candidates,
suppose each screening mean has
zero-mean sub-Gaussian error with proxy $\sigma_{\rm scr}^2/m_{\rm scr}$.
For $\beta>0$, reusing screening for inheritance obeys
\begin{equation}
 \Pr(\exists k:V_k\leq0,\widehat V^{\rm scr}_k>\beta)
 \leq K\exp\!\left(-\frac{m_{\rm scr}\beta^2}
                         {2\sigma_{\rm scr}^2}\right).
 \label{eq:reuse-bound}
\end{equation}
Let $\mathcal F_t^{\rm pre}$ contain the archive, exploration, screening,
material selection, recipient attempts, verification groups, fixed training
trajectories, and completed absorption. If confirmation of the fixed final
proposal uses fresh probes with conditionally
zero-mean sub-Gaussian error of proxy $\sigma_{\rm con}^2/m_{\rm con}$, then
\begin{equation}
 \Pr(\widehat V^{\rm con}(S_t^+)>\beta,V^+\leq0
       \mid\mathcal F_t^{\rm pre})
 \leq \exp\!\left(-\frac{m_{\rm con}\beta^2}
                        {2\sigma_{\rm con}^2}\right).
 \label{eq:fresh-bound}
\end{equation}
\end{theorem}
Across adaptive rounds, per-round budgets $\delta_t$ with
$\sum_t\delta_t\leq\delta$ give a simultaneous bound by choosing the
corresponding positive margins. The reported zero-threshold gate trades this
positive-margin guarantee for greater acceptance sensitivity.

Regret within the realized candidate set is then at most twice the screening error plus the absorption loss, the confirmation error and the margin, and No Recovery has zero absorption loss,
as Corollary~\ref{cor:regret} in Appendix~\ref{app:regret-adaptive} states.

%% file: sections/experiments.tex
\section{Experiments}
\label{sec:experiments}

We first test whether task selection predicts learning and
whether \method{} repairs inheritance, then evaluate the complete loop on four
backbones, and finally measure the value of experience, each component, and the
cost of construction.

\begin{figure}[t]
  \centering
  \includegraphics[width=\textwidth]{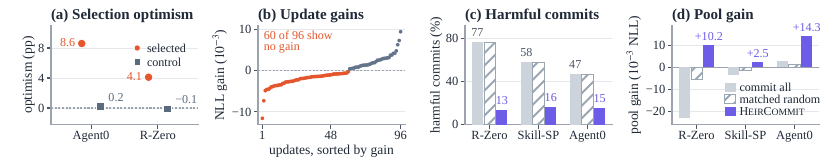}
  \caption{\textbf{Selection does not predict learning, and \method repairs
  inheritance.} (a) Selected-target optimism, the support a target shows on
  the responses that selected it minus its support on fresh responses, against
  fixed-target controls. (b) Correct-target NLL gain of 96 real Skill-SP updates
  on exact-valid tasks, sorted by gain. (c,d) Harmful share of commits and
  pool-mean NLL gain over 64 native updates per host. Matched random commits
  \method's fraction of updates at random.
  Appendix~\ref{app:selection-inheritance-audits} gives the full audit.}
  \label{fig:keyresults}
\end{figure}

\subsection{Task Selection Does Not Predict Learning}

Reusing responses both to select a target and to score it inflates the
target's support. Selected-target optimism, the drop from a candidate's score
on its selection responses to its score on fresh responses, is
\AuditAgentSelected{} pp in Agent0 and \AuditRZeroSelected{} pp in R-Zero, as
Figure~\ref{fig:keyresults}a shows, while fixed-target controls stay near zero
at \AuditAgentControl{} and $\AuditRZeroControl{}$ pp. Under fresh responses,
\AgentFalseFrontier\% of Agent0's and \RZeroFalseFrontier\% of R-Zero's
admitted tasks also leave the native frontier. Even exact-valid tasks need not
teach. In Figure~\ref{fig:keyresults}b, \AuditNonpositiveCount{} of
\AuditUpdateCount{} real Skill-SP updates (\AuditNonpositivePercent\%) do not
reduce correct-target NLL.

\subsection{\method Cuts Harmful Commits on Every Host}

A second, coverage-matched audit takes 64 native updates per host and scores
them on a pool disjoint from the decision probes. \method commits 23--52\% of
the updates and lowers the harmful share of commits from 47--77\% to
13--16\% on all three hosts, as Figure~\ref{fig:keyresults}c shows. Matched
random commits the same fraction at random and keeps the native harmful share,
so \method's cut is 68--83\% relative to it. \method also turns negative pool
gain positive on R-Zero and Skill-SP and raises Agent0's gain fivefold in
Figure~\ref{fig:keyresults}d. It retains 78--87\% of the useful updates, as
Appendix~\ref{app:selection-inheritance-audits} reports.

\subsection{\system Leads on Every Backbone}

\begin{table}[t]
  \centering
  \papertablefont
  \caption{\textbf{Ten-round self-evolution across four backbones (\%).} Each
  average weights its benchmarks equally, and $\pm$ is the sample SD of Math
  Avg.\ over three seeds. ``vs.\ Agent0'' is the relative Math Avg.\ gain of
  \system (Full) over Agent0. Figure~\ref{fig:benchmark-heatmap} and
  Appendix~\ref{app:extended-results} give per-benchmark results.}
  \label{tab:self-evolution}
  \input{tables/self_evolution_compact}
\end{table}

After ten rounds, \system (Full) has the highest Math Avg.\ and General Avg.\
on every backbone of Section~\ref{sec:setup}, as Table~\ref{tab:self-evolution}
shows. It raises the base model's Math Avg.\ by 7.9--30.1\% and exceeds Agent0,
the strongest baseline, by 1.59--2.40 pp (6.0--7.0\%) on the three base models.
\texttt{Gemma-4-E4B-it}, the only instruction-tuned start, leaves the least
headroom, and Full still leads by 0.55 pp.

Across all twelve benchmarks, Full
beats Agent0 on 45 of 48 benchmark--backbone pairs in
Figure~\ref{fig:benchmark-heatmap}, and the three exceptions are 30-problem
AIME sets. No Recovery keeps exploration, screening, and confirmation but
removes absorption. Full exceeds it on every backbone, which isolates Recovery
as a module, and Figure~\ref{fig:math-trajectories} shows Full ending the ten
rounds highest on all four backbones.

These are fixed-round comparisons. A Full round takes 8--9 wall-clock hours on two 96-GB GPUs, as
Appendix~\ref{app:compute-performance} details, and
Section~\ref{sec:exp-budget} compares allocations at equal accounting budget.

\begin{figure}[t]
  \centering
  \includegraphics[width=\textwidth]{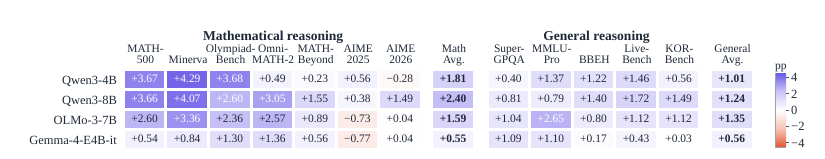}
  \caption{\textbf{Per-benchmark gain of \system (Full) over Agent0 at round
  10} (pp, three-seed means). 45 of 48 cells are positive, and the three
  negative cells are 30-problem AIME sets.}
  \label{fig:benchmark-heatmap}
\end{figure}

\begin{figure}[t]
  \centering
  \includegraphics[width=\textwidth]{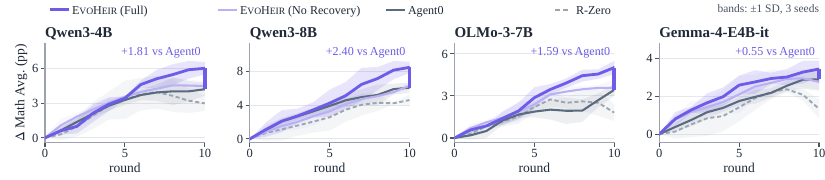}
  \caption{\textbf{Ten-round Math Avg.\ trajectories,} as change from round 0
  (pp). Lines and bands show the mean and $\pm1$ SD over three seeds. The bar
  at round 10 marks \system's lead over Agent0, and each panel has its own y
  scale.}
  \label{fig:math-trajectories}
\end{figure}

\subsection{Experience Outlives Its Model}
\label{sec:experience-value}

The study runs 16 donor interventions on four fixed recipients, from rounds 3,
5, 7, and 10, with two current and two archived donors each. Every intervention
restarts its recipient and is evaluated on the same 1,024 items before
confirmation, and gains are relative to that recipient.

\paragraph{Donor state and experience value diverge.} Eight of the ten donors
weaker than the current root still improve the recipient, as
Figure~\ref{fig:experience-value}a shows. Their mean donor gain is $-1.09$ pp,
but their mean absorption gain is $+0.71$ pp.

\paragraph{Recovery learns and retains.} From the same 64 references, Cache SFT
gains 2.08 pp of new successes but loses 2.61 pp, while Recovery gains 2.54 pp
and loses 1.15 pp, for net gains of $-0.54$ and $+1.39$ pp in
Figure~\ref{fig:experience-value}b. Recovery beats Cache SFT at all four
checkpoints, by 1.93 pp on average, although its own net gain is negative at
round 7 ($-0.20$ pp).

\paragraph{Experience outlives the model that produced it.} At a fixed budget of
64 references, current-only material yields $+0.39$ pp, adding committed-round
history $+0.73$ pp, and admitting all qualified history, including rounds whose
model was rejected, $+1.39$ pp, as Figure~\ref{fig:experience-value}c shows.
A matched replacement then fixes the record count. Three seed groups with four
late-round recipients each swap 32 qualified records from non-inherited
sources for same-age records from inherited sources. The non-inherited records
still add 0.34--0.73 pp at ages one to four, with pointwise 95\% bootstrap
intervals above zero, as Figure~\ref{fig:experience-value}d shows. The gain
therefore comes from the sources, not from more references.

\begin{figure}[t]
  \centering
  \includegraphics[width=\textwidth]{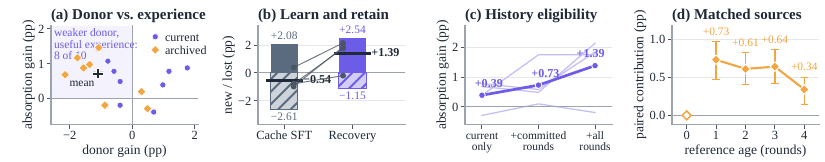}
  \caption{\textbf{Experience value under fixed-recipient interventions}
  (pp, 1,024 items each, before confirmation). (a) Donor gain against
  absorption gain for 16 donors. The shaded quadrant holds weaker donors whose
  experience helps, and $+$ marks their mean. (b) Mean new (up) and lost (down)
  successes from the same 64 references, with black bars at the mean net gain
  and thin lines pairing the two methods at each recipient. (c) Absorption gain
  as history eligibility widens at a fixed 64-reference budget, with thin lines
  for the four recipients. (d) Paired contribution of non-inherited over
  same-age inherited history at 64 records, with pointwise 95\% hierarchical
  bootstrap intervals over three seed groups. Age zero is an identical-material
  control.}
  \label{fig:experience-value}
\end{figure}

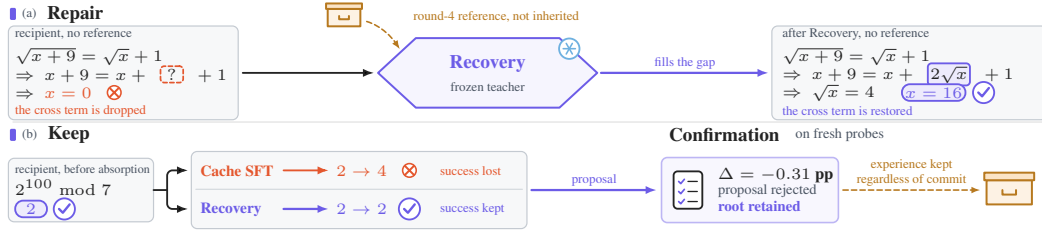
\begin{figure}[t]
  \centering
  \input{figures/qualitative-tikz}
  \caption[Qualitative analysis: one round]{\textbf{Qualitative analysis on
  \texttt{Qwen3-4B-Base} at round 5.} (a) A round-4 archived reference from a
  proposal that was not inherited recovers $x=16$ after the recipient drops the
  cross term; Recovery reproduces it without the reference. (b) Under the same
  64-reference budget, Recovery keeps the remainder that Cache SFT loses. The
  fresh-probe gain is $-0.31$ pp, so \method{} retains the root.}
  \label{fig:qualitative}
\end{figure}

\paragraph{Qualitative analysis.} On \texttt{Qwen3-4B-Base} at round 5, the
recipient drops the cross term in $\sqrt{x+9}=\sqrt{x}+1$ and returns $x=0$.
A round-4 archived reference from a proposal that was not inherited recovers
$x=16$, and Recovery reproduces it without the reference. With the same 64
references, Cache SFT changes the correct remainder of $2^{100}\bmod 7$ from
$2$ to $4$, while Recovery keeps $2$. The fresh-probe confirmation gain is
$-0.31$ pp, so \method{} retains the root (Figure~\ref{fig:qualitative}).

\subsection{Targeted Reuse and Confirmation Both Matter}

\begingroup
\setlength{\intextsep}{4pt}
\setlength{\columnsep}{10pt}
\begin{wraptable}{r}{0.47\textwidth}
  \centering
  \papertablefont
  \caption{\textbf{Component ablations} (\%, mean\,{\tiny$\pm$}\,SD over
  three seeds).}
  \label{tab:core-ablations}
  \input{tables/core_ablations_pm}
\end{wraptable}

Removing any one mechanism from Full lowers both averages, as
Table~\ref{tab:core-ablations} shows. Each ablation keeps the native budget,
round count, objective, and absorption cap. Same-Lineage removes
cross-lineage sources, Random Experience replaces recipient-conditioned
selection with uniform sampling, and Current-Round Only removes history.
These three lower Math Avg.\ by 0.97--1.28 pp and General Avg.\ by
0.48--0.65 pp. NoConfirm removes independent confirmation and has the largest
drops, 2.02 and 1.19 pp. Targeted experience reuse and a final inheritance
decision therefore both contribute. All ablations use the
\texttt{Qwen3-4B-Base} backbone.
\par
\endgroup

\subsection{Equal-Budget Comparison and Sensitivity}
\label{sec:exp-budget}

\begin{figure}[t]
  \centering
  \includegraphics[width=\textwidth]{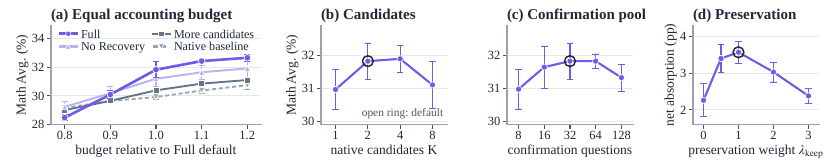}
  \caption{\textbf{Equal-budget comparison and sensitivity on
  \texttt{Qwen3-4B-Base}.} (a--c) Math Avg.\ after ten rounds, mean $\pm$ SD
  over three training seeds. (a) Four methods at equal accounting budgets,
  relative to Full's default of 123.7 units. (b,c) Candidate count and
  confirmation-pool size at budget 1.0, where native training absorbs the
  remaining budget. Open rings mark the defaults, which reuse the main runs.
  (d) Net absorption gain of a fixed recipient on 2,400 problems as the
  preservation weight changes, over three absorption-order seeds. Accounting
  units are not measured GPU-hours or FLOPs.}
  \label{fig:budget-sensitivity}
\end{figure}

A separate study on \texttt{Qwen3-4B-Base} compares methods at equal
cost. A normalized accounting budget charges a single-stream native baseline
100.0 units and Full's default configuration 123.7 units, a 23.7\% overhead
spent on screening, recipient attempts, verification, teacher passes,
absorption, confirmation, and state export. Every variant is trained
separately for ten rounds with three seeds, and native training absorbs the
budget that construction leaves.

Full has the lowest mean at 0.8 of its
default budget but the highest from 1.0 onward, as
Figure~\ref{fig:budget-sensitivity}a shows, so its advantage depends on
leaving enough native training after construction and confirmation are paid
for. At budget 1.0, four candidates give the largest mean, close to the default
of two, and confirmation pools of 32 and 64 questions are also close. The
default preservation weight $\lambda_{\rm keep}=1$ gives the largest net
absorption gain in Figure~\ref{fig:budget-sensitivity}d. Larger settings
consume native-training budget and need not improve the final learner.

%% file: tables/self_evolution_compact.tex
\setlength{\tabcolsep}{2.2pt}
\renewcommand{\arraystretch}{1.08}
\begin{tabular*}{\textwidth}{@{\extracolsep{\fill}}l cccccc ccccc@{}}
\toprule
 & \multicolumn{6}{c}{Math Avg.\ (7 benchmarks) $\uparrow$} & \multicolumn{5}{c}{General Avg.\ (5 benchmarks) $\uparrow$} \\
\cmidrule(lr){2-7}\cmidrule(l){8-12}
Backbone & Base & R-Zero & Agent0 & No Rec. & \textbf{\system} & vs.\ Agent0 & Base & R-Zero & Agent0 & No Rec. & \textbf{\system} \\
\midrule
Qwen3-4B-Base & 25.81 & 28.79\,{\tiny$\pm$0.65} & 30.01\,{\tiny$\pm$0.57} & 30.29\,{\tiny$\pm$0.40} & \textbf{31.82}\,{\tiny$\pm$0.55} & $+$6.0\% & 27.00 & 28.82 & 29.51 & 29.82 & \textbf{30.52} \\
Qwen3-8B-Base & 28.06 & 32.65\,{\tiny$\pm$0.65} & 34.11\,{\tiny$\pm$0.63} & 34.35\,{\tiny$\pm$0.38} & \textbf{36.51}\,{\tiny$\pm$0.67} & $+$7.0\% & 32.34 & 33.84 & 35.24 & 35.75 & \textbf{36.48} \\
OLMo-3-7B & 19.92 & 21.73\,{\tiny$\pm$0.61} & 23.35\,{\tiny$\pm$0.44} & 23.50\,{\tiny$\pm$0.56} & \textbf{24.94}\,{\tiny$\pm$0.48} & $+$6.8\% & 18.05 & 18.66 & 20.05 & 20.34 & \textbf{21.40} \\
Gemma-4-E4B-it & 43.77 & 45.11\,{\tiny$\pm$0.50} & 46.66\,{\tiny$\pm$0.59} & 46.53\,{\tiny$\pm$0.42} & \textbf{47.21}\,{\tiny$\pm$0.44} & $+$1.2\% & 38.88 & 39.01 & 40.37 & 40.73 & \textbf{40.93} \\
\bottomrule
\end{tabular*}

%% file: figures/qualitative-tikz.tex
\input{figures/eh-palette}
\input{figures/teaser-icons}
\def\ehSlotQ{\tikz[baseline=(q.base),line cap=round]{\node[draw=ehHarm,dashed,dash pattern=on 1.6pt off 1pt,line width=.65pt,rounded corners=1.4pt,inner sep=1pt,minimum width=.30cm,minimum height=.19cm,font=\rmfamily\bfseries\fontsize{5.6}{6.2}\selectfont](q){$?$};}}
\def\ehSlotF{\tikz[baseline=(f.base),line cap=round]{\node[draw=ehAccent,fill=ehAccent!12,line width=.7pt,rounded corners=1.4pt,inner sep=1pt,minimum width=.42cm,minimum height=.21cm,font=\rmfamily\fontsize{5.9}{6.6}\selectfont](f){$2\sqrt{x}$};}}
\def\ehTick#1{\draw[tick] (#1) circle (.14);\draw[draw=ehAccent,line width=.9pt] ([xshift=-.07cm]#1)--([xshift=-.02cm,yshift=-.05cm]#1)--([xshift=.07cm,yshift=.06cm]#1);}
\def\ehCross#1{\draw[cross] (#1) circle (.09);\draw[draw=ehHarm,line width=.75pt] ([xshift=-.06cm,yshift=-.06cm]#1)--([xshift=.06cm,yshift=.06cm]#1) ([xshift=-.06cm,yshift=.06cm]#1)--([xshift=.06cm,yshift=-.06cm]#1);}
\begin{tikzpicture}[x=1cm,y=1cm,text=ehInk,line cap=round,line join=round,
  font=\rmfamily\fontsize{6.0}{7.0}\selectfont,
  head/.style={font=\rmfamily\bfseries\fontsize{7.2}{8}\selectfont,anchor=west,inner sep=0pt,outer sep=0pt},
  note/.style={font=\rmfamily\fontsize{5.4}{6.3}\selectfont,text=ehInk2,anchor=west,inner sep=0pt,outer sep=0pt},
  tiny/.style={font=\rmfamily\fontsize{4.9}{5.8}\selectfont,text=ehInk2,anchor=west,inner sep=0pt,outer sep=0pt},
  mathline/.style={anchor=west,inner sep=0pt,outer sep=0pt,font=\rmfamily\fontsize{6.2}{7.2}\selectfont},
  flow/.style={-{Latex[length=1.5mm,width=1.1mm]},draw=ehInk,line width=.7pt},
  flowacc/.style={flow,draw=ehAccent,line width=.85pt},
  flowharm/.style={flow,draw=ehHarm},
  arch/.style={-{Latex[length=1.2mm,width=.9mm]},draw=ehArchiveInk,line width=.6pt,
               dash pattern=on 2pt off 1.2pt},
  pill/.style={rounded corners=3pt,line width=.7pt},
  tick/.style={draw=ehAccent,fill=white,line width=.7pt},
  cross/.style={draw=ehHarm,fill=white,line width=.6pt},
  statebox/.style={rounded corners=2.2pt,draw=ehSlate!36,line width=.45pt,fill=ehSlate!5},
  verdictbox/.style={rounded corners=2.2pt,draw=ehAccent!45,line width=.6pt,fill=ehAccent!5}]
\path[use as bounding box] (0.05,0.80) rectangle (13.95,3.90);

\draw[statebox] (0.06,2.28) rectangle (3.10,3.52);
\draw[statebox] (10.16,2.28) rectangle (13.80,3.52);
\draw[statebox] (0.06,0.93) rectangle (1.98,1.77);
\draw[statebox] (2.48,0.84) rectangle (6.92,1.86);
\draw[draw=ehSlate!26,line width=.45pt] (2.54,1.36)--(6.86,1.36);
\draw[verdictbox] (8.70,0.92) rectangle (11.04,1.78);

\fill[ehAccent,rounded corners=.6pt] (0.08,3.60) rectangle (0.16,3.76);
\node[tiny] (hA) at (0.24,3.68) {(a)};
\node[head] (tA) at ([xshift=4pt]hA.east) {Repair};

\node[tiny] at (0.15,3.42) {recipient, no reference};
\node[mathline] at (0.15,3.12) {$\sqrt{x+9}=\sqrt{x}+1$};
\node[mathline] at (0.15,2.88) {$\Rightarrow\ x+9=x+\;\,\ehSlotQ\;\,+1$};
\node[mathline] (eqL3) at (0.15,2.64) {$\Rightarrow\ \textcolor{ehHarm}{x=0}$};
\node[coordinate] (crL) at ([xshift=0.30cm]eqL3.east) {};
\ehCross{crL}
\node[tiny,text=ehHarm] at (0.15,2.40) {the cross term is dropped};
\draw[flow] (3.16,2.90)--(4.90,2.90);

\draw[fill=ehAccent!6,draw=ehAccent,line width=.75pt]
  (4.95,2.90)--(5.45,3.36)--(7.35,3.36)--(7.85,2.90)--(7.35,2.44)--(5.45,2.44)--cycle;
\node[head,anchor=center,text=ehAccent] at (6.40,3.02) {Recovery};
\node[tiny,anchor=center] at (6.40,2.72) {frozen teacher};
\ehBadgeSnow{7.48}{3.22}{.9}
\ehArchive{4.50}{3.66}{.72}
\draw[arch] (4.80,3.58) .. controls (5.00,3.52) and (5.14,3.36) .. (5.24,3.21);
\node[tiny,text=ehArchiveInk] at (5.42,3.64) {round-4 reference, not inherited};
\draw[flowacc] (7.90,2.90)--(10.12,2.90);
\node[tiny,text=ehAccent,anchor=south] at (9.01,2.95) {fills the gap};

\node[tiny] at (10.30,3.42) {after Recovery, no reference};
\node[mathline] at (10.30,3.12) {$\sqrt{x+9}=\sqrt{x}+1$};
\node[mathline] at (10.30,2.88) {$\Rightarrow\ x+9=x+\;\,\ehSlotF\;\,+1$};
\node[mathline] (eqR3) at (10.30,2.64) {$\Rightarrow\ \sqrt{x}=4$};
\node[coordinate] (pinA) at ([xshift=0.78cm]eqR3.east) {};
\draw[pill,fill=ehAccent!12,draw=ehAccent] (pinA) ++(-0.42,-0.11) rectangle ++(0.84,0.22);
\node[note,anchor=center,text=ehAccent,font=\rmfamily\bfseries\fontsize{5.9}{6.8}\selectfont] at (pinA) {$x=16$};
\node[coordinate] (chkA) at ([xshift=0.66cm]pinA) {};
\ehTick{chkA}
\node[tiny,text=ehAccent] at (10.30,2.40) {the cross term is restored};

\draw[black!12,line width=.45pt] (0.12,2.25)--(13.88,2.25);

\fill[ehAccent,rounded corners=.6pt] (0.08,2.00) rectangle (0.16,2.16);
\node[tiny] (hB) at (0.24,2.08) {(b)};
\node[head] (tB) at ([xshift=4pt]hB.east) {Keep};

\node[tiny,font=\rmfamily\fontsize{4.6}{5.5}\selectfont] at (0.15,1.61) {recipient, before absorption};
\node[mathline] at (0.15,1.39) {$2^{100}\bmod 7$};
\draw[pill,fill=ehAccent!12,draw=ehAccent] (0.15,0.99) rectangle (0.57,1.21);
\node[note,anchor=center,text=ehAccent,font=\rmfamily\bfseries\fontsize{5.9}{6.8}\selectfont] at (0.36,1.10) {$2$};
\node[coordinate] (ck0) at (0.80,1.10) {};
\ehTick{ck0}
\draw[ehInk,line width=.7pt] (1.98,1.35)--(2.16,1.35);
\draw[flow,rounded corners=2pt] (2.16,1.35)|-(2.44,1.61);
\draw[flow,rounded corners=2pt] (2.16,1.35)|-(2.44,1.10);

\node[note,text=ehHarm,font=\rmfamily\bfseries\fontsize{5.8}{6.7}\selectfont] at (2.60,1.61) {Cache SFT};
\draw[flowharm] (3.70,1.61)--(4.30,1.61);
\node[mathline,text=ehHarm] (res1) at (4.40,1.61) {$2\to 4$};
\node[coordinate] (cr1) at ([xshift=0.30cm]res1.east) {};
\ehCross{cr1}
\node[tiny,text=ehHarm] at (5.77,1.61) {success lost};

\node[note,text=ehAccent,font=\rmfamily\bfseries\fontsize{5.8}{6.7}\selectfont] at (2.60,1.10) {Recovery};
\draw[flowacc] (3.70,1.10)--(4.30,1.10);
\node[mathline,text=ehAccent] (res2) at (4.40,1.10) {$2\to 2$};
\node[coordinate] (ck2) at ([xshift=0.30cm]res2.east) {};
\ehTick{ck2}
\node[tiny,text=ehAccent] at (5.77,1.10) {success kept};

\draw[flowacc] (6.98,1.35)--(8.64,1.35);
\node[tiny,text=ehAccent,anchor=south] at (7.81,1.40) {proposal};

\node[head] (tC) at (8.80,2.08) {Confirmation};
\node[note] at ([xshift=6pt]tC.east) {on fresh probes};
\ehProbes{9.05}{1.35}{.95}
\node[note,font=\rmfamily\bfseries\fontsize{6.2}{7.1}\selectfont,text=ehInk] at (9.45,1.55) {$\Delta=-0.31$ pp};
\node[note] at (9.45,1.34) {proposal rejected};
\node[note,text=ehAccent,font=\rmfamily\bfseries\fontsize{5.6}{6.5}\selectfont] at (9.45,1.14) {root retained};
\draw[arch] (11.10,1.35)--(12.92,1.35);
\ehArchive{13.30}{1.36}{1.0}
\node[tiny,text=ehArchiveInk,anchor=south,align=center] at (12.01,1.40) {experience kept\\regardless of commit};
\end{tikzpicture}

%% file: tables/core_ablations_pm.tex
\setlength{\tabcolsep}{4pt}
\renewcommand{\arraystretch}{1.10}
\begin{tabular}{@{}lcc@{}}
\toprule
\textbf{Variant} & \textbf{Math Avg.} $\uparrow$ & \textbf{General Avg.} $\uparrow$ \\
\midrule
\textbf{\system (Full)} & \textbf{\AblFullMath}\,{\tiny$\pm$\AblFullMathSpread} & \textbf{\AblFullGeneral}\,{\tiny$\pm$\AblFullGeneralSpread} \\
\midrule
Same-Lineage & \AblLineageMath\,{\tiny$\pm$\AblLineageMathSpread} & \AblLineageGeneral\,{\tiny$\pm$\AblLineageGeneralSpread} \\
Random Experience & \AblRandomMath\,{\tiny$\pm$\AblRandomMathSpread} & \AblRandomGeneral\,{\tiny$\pm$\AblRandomGeneralSpread} \\
Current-Round Only & \AblCurrentMath\,{\tiny$\pm$\AblCurrentMathSpread} & \AblCurrentGeneral\,{\tiny$\pm$\AblCurrentGeneralSpread} \\
NoConfirm & \AblNoConfirmMath\,{\tiny$\pm$\AblNoConfirmMathSpread} & \AblNoConfirmGeneral\,{\tiny$\pm$\AblNoConfirmGeneralSpread} \\
\bottomrule
\end{tabular}

%% file: sections/limitations_conclusion.tex
\section{Conclusion}

We identify the selection-to-update gap: admitting a
task that looks learnable does not ensure a useful update, and
\AuditNonpositiveCount{} of \AuditUpdateCount{} updates on verified tasks show
no gain. \method inherits a successor only when it beats the current model on
fresh probes and cuts the harmful share of commits by 68--83\% on
three self-evolution hosts. \system separates the parameters a round inherits
from the experience it keeps and raises the math average by up to
7.0\% over Agent0, the strongest baseline. Self-evolving systems
should judge the updates they produce, not only the tasks they choose, and keep
experience beyond the models that produced it.

%% file: appendix/experimental_setup.tex
\section{Experimental Setup and Hyperparameters}
\label{app:experimental-setup}

\paragraph{Backbones and protocol.}
The checkpoints are \texttt{Qwen/Qwen3-4B-Base}, \texttt{Qwen/Qwen3-8B-Base},
\texttt{allenai/Olmo-3-1025-7B}, and \texttt{google/gemma-4-E4B-it}. The two
Qwen checkpoints and OLMo are pretrained base models; Gemma is
instruction-tuned. Every backbone and method runs ten evolution rounds with
three training seeds, with evaluation after every round. The evaluated
checkpoint is the learner actually inherited at that round, $S_{t+1}$,
including rounds that retain $S_t$. In the host-system \method studies,
Skill-SP, R-Zero, and Agent0 supply native updates and an integration adapter
governs inheritance.

\paragraph{Native exploration and screening.}
\label{app:native-hyperparameters}
Each round makes 1,024 task-generation attempts in one batch and uses ten
learner responses per question for native admission. Two isolated native
branches are then trained from the same root, each on a cohort of 64
questions, with the recipe in Table~\ref{tab:optimization-settings}. Native
training updates all parameters of the Curriculum Scout and Evolving Learner
with AdamW~\citep{loshchilov2019decoupledweightdecayregularization}, and
Recovery updates LoRA parameters~\citep{hu2021loralowrankadaptationlarge}.
Screening scores the root and the two branch endpoints on 16 questions from
$P_t^{\rm scr}$, and confirmation scores the final proposal on 32 fresh
questions from $P_t^{\rm con}$, with four responses per question and state.
Screening ties favor the root, then the branches in a fixed order.

\newsavebox{\optimizationtablebox}
\begin{table}[!htbp]
  \centering
  \papertablefont
  \caption{\textbf{Optimization settings for \system (Full).}
  Native learner settings are per branch. Batch sizes distinguish prompt
  groups from individual responses.}
  \label{tab:optimization-settings}
  \setlength{\tabcolsep}{4pt}
  \renewcommand{\arraystretch}{1.08}
  \begin{lrbox}{\optimizationtablebox}
  \begin{tabular}{@{}lccc@{}}
    \toprule
    Setting & \shortstack{Curriculum\\Scout} & \shortstack{Evolving Learner\\(native training)} & \shortstack{Evolving Learner\\(Recovery)} \\
    \midrule
    Trainable parameters & Full model & Full model & LoRA \\
    Optimizer & AdamW & AdamW & AdamW \\
    Learning rate & $10^{-6}$ & $10^{-6}$ & $5\times10^{-5}$ \\
    Weight decay & $0.01$ & $0.01$ & $0$ \\
    Per-round optimization budget & 5 calls & 15 calls & $\leq16$ updates \\
    Questions per batch & 16 & 8 & $\leq4$ \\
    Responses per question & 5 & 5 & 1 fixed response \\
    Effective batch (responses) & 80 & 40 & $\leq4$ \\
    Per-device microbatch (responses) & 1 & 1 & 1 \\
    GPUs per stage & 1 & 1 & 1 \\
    LoRA rank / scaling $\alpha$ & N/A & N/A & 8 / 16 \\
    LoRA target projections & N/A & N/A & \texttt{q\_proj}, \texttt{v\_proj} \\
    \bottomrule
  \end{tabular}
  \end{lrbox}
  \usebox{\optimizationtablebox}
  \par\smallskip
  \begin{minipage}{\wd\optimizationtablebox}
    \raggedright
    \papertablefont The native scheduler uses zero learning rate on its first
    optimizer call and $10^{-6}$ thereafter. Native budgets count optimizer
    calls, including this initial call, rather than nonzero parameter updates.
  \end{minipage}
\end{table}

\paragraph{Consensus labels and probe qualification.}
\label{app:consensus-protocol}
The round-start Evolving Learner $S_t$ generates ten admission responses
per question using the settings in Appendix~\ref{app:generation-evaluation}.
In the first round, $S_t$ is the initial checkpoint; subsequent rounds use the
actually inherited learner. The answer comparator clusters successfully
extracted, nonempty answers. If $m_{\rm valid}$ responses yield such answers
and the largest cluster contains $c_{\max}$ responses, consensus support is
$c_{\max}/m_{\rm valid}$. Questions with $m_{\rm valid}=0$ are excluded.
The representative answer of the largest cluster becomes the pseudo-label;
probes require support at least $0.7$. Ties select the first encountered
largest cluster, but tied maxima have support at most $0.5$ and cannot
qualify as probes. Each role's question--label file is frozen and shared
across all states evaluated on that role's pool.
Confirmation questions and labels are fixed before candidate construction and
excluded from native training, screening, and Recovery. Screening and
confirmation use fresh scoring responses, seeds, and caches, separate from the
admission responses that form the labels. \method adopts $S_t^+$ only when its
confirmation gain is strictly positive, that is, $\tau_t=0$ in
Eq.~\ref{eq:heircommit}.

\paragraph{Generation settings.}
\label{app:generation-evaluation}
The Curriculum Scout's native rollouts use temperature $1.0$, top-$p=0.99$,
and five responses per question, and offline candidate generation uses
temperature $1.0$ and top-$p=0.95$. Learner generation is capped at 4,096
response tokens with top-$k$ filtering disabled unless stated otherwise.
Native learner training samples five responses at temperature $1.0$ and
top-$p=0.99$. Native admission uses ten responses at temperature $1.0$,
top-$p=1$, and top-$k=40$. Screening and confirmation use four responses at
temperature $0.7$ and top-$p=0.95$, and recipient attempts use one greedy
response. Appendix~\ref{app:extended-results} gives the benchmark decoding and
metrics.

%% file: appendix/detailed_method.tex
\section{Detailed Method and Implementation}
\label{app:detailed-method}

\subsection{State Construction and Recipient Selection}
\label{app:state-construction}

The resumable learner $S_t=(\theta_t,\omega_t)$ holds the model parameters and
the mutable training state, including optimizer moments, the learning-rate
scheduler, and random generators. The Evolution Controller runs each admitted
cohort in an isolated copy of the same root, and each branch follows the host's
training objective. The controller freezes the candidate set
$\mathcal C_t^{\rm nat}\cup\{S_t\}$ before screening, selects the recipient by
paired root-relative gain on $P_t^{\rm scr}$, and breaks ties in favor of the
root, followed by a fixed branch order. Algorithm~\ref{alg:evoheir} lists one
round.

\begin{algorithm}[t]
\caption{\system at round $t$}
\label{alg:evoheir}
\begin{algorithmic}[1]
\Require root $S_t$, controller $A_t$, archive $M_t$, budget, fresh probe pools
\State $b_{1:J}\gets\Call{Scout}{S_t,A_t}$
\State $\mathcal C_t^{\rm nat},E_t\gets
 \Call{NativeExplore}{S_t,b_{1:J}}$
 \Comment{isolated states and train-role responses}
\State Freeze $\mathcal C_t^{\rm nat}\cup\{S_t\}$
\State $S_t^{\rm sel}\gets\Call{Screen}{S_t,\mathcal C_t^{\rm nat},P_t^{\rm scr}}$
\State $S_t^+\gets S_t^{\rm sel}$
\If{Recovery is enabled}
  \State $D_t\gets\Call{QualifyMaterial}{E_t,M_t}$
  \State $Y_t\gets\Call{RecipientAttempts}{S_t^{\rm sel},D_t}$
   \Comment{exact-checkpoint reuse or one rollout batch}
  \State $D_t^{\rm repair},D_t^{\rm keep}\gets\Call{VerifyAndGroup}{D_t,Y_t}$
  \State Select from both groups within the predeclared absorption budget
  \State Freeze recipient teacher, trajectories, and verification groups
  \State $S_t^+\gets\Call{Absorb}{S_t^{\rm sel},D_t^{\rm repair},D_t^{\rm keep}}$
   \Comment{identity if no eligible experience}
\EndIf
\State Freeze $S_t^+$
\State $v\gets\Call{PairedGain}{S_t,S_t^+,P_t^{\rm con}}$
 \Comment{zero for an identical state}
\State $S_{t+1}\gets S_t^+$ if $v>\tau_t$, otherwise $S_t$
\State $M_{t+1}\gets\Call{UpdateArchive}{M_t,E_t}$
 \Comment{qualified train-role records, regardless of commit}
\State $A_{t+1}\gets
 \Call{Observe}{A_t,\operatorname{source}(S_t^{\rm sel}),v}$
\State Freeze $S_{t+1}$ before sealed-final evaluation; \Return $S_{t+1}$
\end{algorithmic}
\end{algorithm}

\subsection{Recipient-Conditioned Recovery}
\label{app:recipient-recovery}
\label{app:conditional-recovery}

Recovery adapts distillation~\citep{hinton2015distillingknowledgeneuralnetwork}
and Learning without Forgetting~\citep{li2017learningforgetting} to
recipient-specific failures and successes. Current records $E_t$ and archived
records $M_t$ supply problems and qualified references, each with its
generating checkpoint, round, and verification trace. A cached recipient
attempt is reused only at the exact selected checkpoint. The same
answer-equivalence rule verifies the attempt and its reference, and membership
in $D_t^{\rm repair}$ or $D_t^{\rm keep}$ is decided for the current recipient,
even when the reference comes from an earlier round. The frozen teacher and the
student both start from $\theta_t^{\rm sel}$ and predict at the same fixed
recipient prefixes throughout absorption, as in Eq.~\ref{eq:recovery-loss}, so
the procedure does not resample from the changing student.

\subsection{Archive, Feedback, and State Inheritance}
\label{app:archive-inheritance}

\method freezes $S_t^+$ before scoring it against the original root on the
fixed confirmation pool with the same probes and frozen labels
(Appendix~\ref{app:consensus-protocol}). It commits the complete proposal when
the paired gain exceeds $\tau_t$, with $\tau_t=0$ in the empirical mode, and
otherwise restores the complete root state. Each round confirms one final
proposal, and later rounds receive fresh evidence. Archive updates follow a
separate path in Algorithm~\ref{alg:evoheir}: qualified current records are
considered for admission regardless of the commit outcome, so a round that
retains the root still contributes references to later rounds. The controller
records the signed confirmation value with recipient and experience provenance,
and subsequent curricula follow the host's feedback rule, shared by Full and
No Recovery.

\paragraph{Failure modes and safeguards.}
An incomplete or conflicting reference is excluded before training, a
duplicate normalized problem occupies one archive slot, and a cached attempt is
reused only from the exact selected checkpoint. A flawed reference that passes
the verifier can still steer the repair teacher, so repair depends on
verification quality. Absorption can also lose successes, as the round-7 net
gain of $-0.20$ pp in Figure~\ref{fig:experience-value} shows. With the
empirical zero threshold, sampling noise can admit a harmful proposal; the
positive-margin mode of Theorem~\ref{thm:commit} bounds that risk.

\subsection{Implementation Configuration}
\label{app:implementation-configuration}

Appendix~\ref{app:experimental-setup} and
Table~\ref{tab:optimization-settings} give the \system (Full) recipe. The
remaining implementation details are as follows.

\begin{description}[font=\normalfont\itshape,style=nextline,leftmargin=0pt,
  labelsep=0pt,itemsep=3pt,topsep=3pt]
\item[Response-token aggregation.]
Response loss is computed only on recorded response tokens; prompt tokens are
masked out. If an EOS token is present in the recorded response, it is included
in the loss. Truncated responses contribute their available response prefix
without synthesizing a missing EOS token. Token-level forward KL over the full
vocabulary is averaged within each response, then over questions, so repair
and keep examples have equal per-question weight.
\item[Absorption and optimizer state.]
Recovery uses AdamW with learning rate $5\times10^{-5}$, $(\beta_1,\beta_2)=(0.9,0.999)$,
epsilon $10^{-8}$, zero weight decay, gradient clipping at $1.0$, and LoRA
dropout $0$. The adapter has rank $8$, scaling $\alpha=16$, and targets
\texttt{q\_proj} and \texttt{v\_proj}; absorption uses at most 16 updates.
After absorption, the adapter is merged into the recipient and the merged
weights are reloaded and checked for exact consistency. If the absorbed model
is accepted, it becomes the next native root and the native optimizer and
scheduler are reinitialized; Recovery optimizer state is not carried into
native training. If the proposal is rejected, the original root is retained
with its parameters, optimizer moments, scheduler, and random state unchanged.
At the next round's native-stage boundary, the native optimizer and scheduler
are reinitialized for whichever model was inherited. Rollback restores the
root snapshot; it does not carry Recovery moments into a later stage.
\item[Archive policy.]
The archive stores at most 256 qualified training records, deduplicated by
normalized problem identity. Conflicting labels, verification conflicts,
known reasoning errors, incomplete responses, and non-training records are
excluded. When capacity is exceeded, records are selected deterministically
using round/source-stratified round-robin admission, with stable hash ordering
within each group.
\item[Recovery budget.]
Each round nominates up to 32 exploration records and 32 learner-success
records, and recipient attempts reuse exact-checkpoint cached responses or use
one additional response within a 64-generation budget. Recovery makes a single
pass with at most 16 optimizer updates, capped at 131,072 response loss tokens
and 262,144 combined teacher and student forward tokens. An empty usable pool
leaves $S_t^{\rm sel}$ unchanged.
\end{description}

%% file: appendix/theoretical_analysis.tex
\section{Theoretical Analysis and Proofs}
\label{app:theoretical-analysis}

\subsection{Setup and Task-Side Non-Identifiability}
\label{app:non-identifiability}

Fix a round and index the frozen native candidates by $1,\ldots,K$, with
index $0$ denoting root. Write $V_0=0$,
$\widehat V_k^{\rm scr}=V_k+\varepsilon_k$, $\varepsilon_0=0$, and
$k^*=\arg\max_{0\leq k\leq K}\widehat V_k^{\rm scr}$, using the tie rule
in Appendix~\ref{app:state-construction}. The completed proposal has value
$V^+=V_t(S_t^+)$.

For screening, $\mathcal G_t$ denotes the history before screening,
including root, the declared distribution $P_t$, and the frozen native
candidates. For confirmation, $\mathcal F_t^{\rm pre}$ contains the full
construction history, including the archive, screening, recipient attempts,
material selection, verification groups, and completed absorption.
Thus $S_t^+$ and $V^+$ are fixed conditional on
$\mathcal F_t^{\rm pre}$. The task-side observation $O_t$ in
Proposition~\ref{prop:identifiability} is the pre-update behavioral
interface defined in Section~\ref{sec:setup}.

\begin{proof}[Proof of Proposition~\ref{prop:identifiability}]
Fix a native step size $\eta>0$ and a value gap $\Delta>0$. At root
$\theta=0$, consider two learning worlds with local training losses and a
common probe loss
\begin{equation}
 \ell_+(\theta)=c+\theta,\qquad
 \ell_-(\theta)=c-\theta,\qquad
 \ell_p(\theta)=c_p+\frac{\Delta}{\eta}\theta.
 \label{eq:app-indistinguishable-losses}
\end{equation}
The worlds share the stipulated root behavior, target validity, and root
loss values. Their training gradients are $+1$ and $-1$, respectively,
and are outside $O_t$. One gradient step reaches $-\eta$ or $+\eta$.
For $q=-\ell_p$, the root-relative values are therefore $+\Delta$ and
$-\Delta$. Positive smooth losses on the whole real line can be obtained
by multiplying each linear term by a smooth cutoff in $[0,1]$, equal to one on
$[-\eta,\eta]$ and zero outside $[-2\eta,2\eta]$, choosing
$c>2\eta$ and $c_p>2\Delta$. This preserves the root gradients and both
post-update probe values.

An observation-based rule uses the same commit probability $p$ in both
worlds. Against the better of root and the native update, its expected
regrets are $(1-p)\Delta$ and $p\Delta$. Their maximum is at least
their average, $\Delta/2$.
\end{proof}

\subsection{Selected-Future Optimism and Independent Confirmation}
\label{app:optimism-confirmation}

\begin{proposition}[Selected-future optimism]
\label{prop:selected-optimism}
Condition on the frozen native candidates and root. If
$\widehat V^{\rm scr}_k=V_k+\varepsilon_k$ with
$\E[\varepsilon_k]=0$ for each $k$, then the screen winner satisfies
$\E[\widehat V^{\rm scr}_{k^*}-V_{k^*}]\geq0$.
\end{proposition}

\begin{proof}[Proof of Proposition~\ref{prop:selected-optimism}]
Condition on $\mathcal G_t$ and assume
$\E[\varepsilon_k\mid\mathcal G_t]=0$ for each fixed candidate.
The maximum is convex, so
\begin{equation}
 \begin{aligned}
 \E[\widehat V_{k^*}^{\rm scr}\mid\mathcal G_t]
 &=\E\!\left[\max_{0\leq k\leq K}(V_k+\varepsilon_k)
       \mid\mathcal G_t\right]\\
 &\geq\max_{0\leq k\leq K}V_k
 \geq\E[V_{k^*}\mid\mathcal G_t].
 \end{aligned}
 \label{eq:app-selected-optimism}
\end{equation}
Subtracting the last expectation gives nonnegative expected optimism.
The argument uses conditional centering of each error and allows
correlations induced by shared probes.
\end{proof}

\begin{proof}[Proof of Theorem~\ref{thm:commit}]
A conditionally centered sub-Gaussian error $X$ with proxy $s^2>0$
satisfies $\E[e^{uX}\mid\mathcal H]\leq e^{u^2s^2/2}$.
For $a>0$, exponential Markov inequality and optimization over $u>0$ give
\begin{equation}
 \Pr(X>a\mid\mathcal H)
 \leq\inf_{u>0}\exp\!\left(-ua+\frac{u^2s^2}{2}\right)
 =\exp\!\left(-\frac{a^2}{2s^2}\right).
 \label{eq:app-subgaussian-tail}
\end{equation}
For a candidate with $V_k\leq0$, the event
$\widehat V_k^{\rm scr}>\beta$ implies $\varepsilon_k>\beta$.
Applying Eq.~\ref{eq:app-subgaussian-tail} with
$s^2=\sigma_{\rm scr}^2/m_{\rm scr}$ and taking a union bound over the
$K$ non-root candidates yields
\begin{equation}
 \Pr\!\left(\exists k\in\{1,\ldots,K\}:V_k\leq0,
       \widehat V_k^{\rm scr}>\beta\mid\mathcal G_t\right)
 \leq K\exp\!\left(-\frac{m_{\rm scr}\beta^2}
                         {2\sigma_{\rm scr}^2}\right).
 \label{eq:app-screening-reuse}
\end{equation}
Root contributes no event because its screening score is exactly zero.
The union bound also applies when candidate errors are correlated.

For confirmation, condition instead on $\mathcal F_t^{\rm pre}$, after
the proposal has been completed. Define
$\xi_t=\widehat V_t^{\rm con}(S_t^+)-V^+$ and assume its conditional
sub-Gaussian proxy is $\sigma_{\rm con}^2/m_{\rm con}$, as in the
theorem. On $\{V^+\leq0\}$, acceptance above $\beta$ implies
$\xi_t>\beta$. Thus
\begin{equation}
 \Pr\!\left(\widehat V_t^{\rm con}(S_t^+)>\beta,\ V^+\leq0
          \mid\mathcal F_t^{\rm pre}\right)
 \leq\exp\!\left(-\frac{m_{\rm con}\beta^2}
                        {2\sigma_{\rm con}^2}\right).
 \label{eq:app-fresh-confirmation}
\end{equation}
This tail bound concerns the single proposal fixed before confirmation.
All selection and absorption decisions are included in the conditioning
history.
\end{proof}

For a prescribed round-level error budget $\delta_t\in(0,1)$, a sufficient
certified threshold is
\begin{equation}
 \tau_t=\sigma_{\rm con}
       \sqrt{\frac{2\log(1/\delta_t)}{m_{\rm con}}}.
 \label{eq:app-certified-margin}
\end{equation}
The threshold and its calibration are fixed before confirmation scoring.
The empirical mode uses $\tau_t=0$, for which the exponential
bound is one, while Eq.~\ref{eq:app-certified-margin} gives the
$\delta_t$-level bound.

For conditionally independent paired observations $X_i=q(S_t^+;z_i)-q(S_t;z_i)\in[-1,1]$ with conditional mean $V^+$, Hoeffding's lemma gives proxy $1/m_{\rm con}$ for their mean, so $\sigma_{\rm con}^2=1$ is valid for exact correctness and consensus matching. Raw negative-NLL utility needs a separate tail assumption.

\subsection{Regret Decomposition and Adaptive Rounds}
\label{app:regret-adaptive}

\begin{corollary}[Regret within the realized candidate set]
\label{cor:regret}
If all native screening errors are at most $\epsilon_{\rm scr}$ in absolute
value and the final proposal's confirmation error is at most
$\epsilon_{\rm con}$, define the absorption loss
$d_{\rm rec}=[V_{k^*}-V^+]_+$. \method with $\tau\geq0$ satisfies
\begin{equation}
 R_t\leq2\epsilon_{\rm scr}+d_{\rm rec}+\epsilon_{\rm con}+\tau.
 \label{eq:regret-bound}
\end{equation}
\end{corollary}

\begin{proof}[Proof of Corollary~\ref{cor:regret}]
Assume the screening errors satisfy
$|\varepsilon_k|\leq\epsilon_{\rm scr}$ and the final confirmation error
satisfies $|\xi_t|\leq\epsilon_{\rm con}$. Let
$\bar k\in\arg\max_{0\leq k\leq K}V_k$, and define
$d_{\rm rec}=[V_{k^*}-V^+]_+$.
Screening optimality gives
\begin{equation}
 \begin{aligned}
 V_{\bar k}
 &\leq\widehat V_{\bar k}^{\rm scr}+\epsilon_{\rm scr}
 \leq\widehat V_{k^*}^{\rm scr}+\epsilon_{\rm scr}\\
 &\leq V_{k^*}+2\epsilon_{\rm scr}
 \leq V^++d_{\rm rec}+2\epsilon_{\rm scr}.
 \end{aligned}
 \label{eq:app-regret-selection}
\end{equation}
The last expression also bounds $V^+$ because
$d_{\rm rec},\epsilon_{\rm scr}\geq0$. Hence it bounds the best value
$V_{\max}=\max\{V_0,\ldots,V_K,V^+\}$ among all constructed states.

If the proposal is accepted, inheritance has value $V^+$ and regret is
at most $d_{\rm rec}+2\epsilon_{\rm scr}$. If it is rejected, inheritance
has value zero and
$\widehat V_t^{\rm con}(S_t^+)\leq\tau$, implying
$V^+\leq\tau+\epsilon_{\rm con}$. For $\tau\geq0$, both cases give
\begin{equation}
 R_t\leq2\epsilon_{\rm scr}+d_{\rm rec}+\epsilon_{\rm con}+\tau.
 \label{eq:app-regret-bound}
\end{equation}
\end{proof}

The absorption term measures decreases from recipient to proposal. It is
zero for No Recovery and whenever absorption has nonnegative
recipient-relative gain.

Across rounds, let $E_t^{\rm bad}$ denote accepting a proposal with
$V_t(S_t^+)\leq0$. For $\delta\in(0,1)$, choose prescribed positive budgets satisfying
$\sum_t\delta_t\leq\delta$. Applying the conditional confirmation bound
in each round, then taking expectations and a union bound, gives
\begin{equation}
 \Pr\!\left(\bigcup_t E_t^{\rm bad}\right)
 \leq\sum_t\E\!\left[
      \Pr(E_t^{\rm bad}\mid\mathcal F_t^{\rm pre})\right]
 \leq\sum_t\delta_t\leq\delta.
 \label{eq:app-adaptive-budget}
\end{equation}
For confirmation rounds indexed by $t\geq1$, the allocation
$\delta_t=\delta/[t(t+1)]$ supports an unbounded horizon because the
series sums to $\delta$. Conditioning permits each round's construction
to depend on earlier outcomes, while its confirmation evidence satisfies
the fresh conditional error assumption.

%% file: appendix/task_selection_inheritance_audits.tex
\section{Audits of task selection and update inheritance}
\label{app:selection-inheritance-audits}

\paragraph{Fresh-response and realized-update audits.}
The task-selection audit freezes the checkpoint, decoder, selected target, and
admission decision before drawing an independent response batch.
Selected-target optimism is $100(\widehat p_{\rm sel}-\widehat p_{\rm fresh})$,
with fresh responses scored against the original target, and the fixed-target
control chooses its target independently of the responses that measure
support. The admission audit counts originally admitted tasks that leave the
native frontier under fresh responses: \AgentFalseFrontier\% for Agent0 and
\RZeroFalseFrontier\% for R-Zero. The exact-valid audit measures 96 realized
Skill-SP optimizer transitions with paired held-out correct-target losses,
holding the root, training batch, optimizer state, step budget, and loss mask
fixed for each comparison. It is separate from the 64-update pools below.

\paragraph{Update inheritance across native sources.}
Table~\ref{tab:cc-full} gives the complete audit behind
Figure~\ref{fig:keyresults}c,d. Commit-all inherits every update, and matched
random commits a random subset at \method's coverage and reports expected
values. Coverage is the committed fraction of updates, useful retention the
fraction of useful updates kept, harmful commits the fraction of accepted
updates with negative gain, and pool gain the mean NLL reduction over the full
pool with a rejection counted as zero.

\begin{table}[!htbp]
  \centering
  \papertablefont
  \caption{\textbf{Complete \method audit across native update sources.}
  Matched random uses \method's coverage within each host. Gain and
  $\Delta$ Native are in $10^{-3}$ NLL units; higher is better.}
  \label{tab:cc-full}
  \input{tables/cc_main}
\end{table}

%% file: tables/cc_main.tex
\setlength{\tabcolsep}{5pt}
\renewcommand{\arraystretch}{1.08}
\begin{tabular}{@{}llcccrr@{}}
\toprule
Host & Persistence rule & Coverage & \shortstack{Useful\\retained} & \shortstack{Harmful\\commits} & \shortstack{Mean gain\\($10^{-3}$ NLL)} & \shortstack{$\Delta$ Native\\($10^{-3}$ NLL)} \\
\midrule
\multirow{3}{*}{Skill-SP}
 & Native commit-all & \SkillHeirNativeCoverage & \SkillHeirNativeUseful & \SkillHeirNativeHarmful & \SkillHeirNativeGainMilli & $0.00$ \\
 & Matched random & \SkillHeirCoverage & \SkillHeirRandomUseful & \SkillHeirRandomHarmful & \SkillHeirRandomGainMilli & \SkillHeirRandomVsNativeMilli \\
 & \method & \SkillHeirCoverage & \SkillHeirUseful & \SkillHeirHarmful & \SkillHeirGainMilli & \SkillHeirDeltaMilli \\
\midrule
\multirow{3}{*}{R-Zero}
 & Native commit-all & \RZeroHeirNativeCoverage & \RZeroHeirNativeUseful & \RZeroHeirNativeHarmful & \RZeroHeirNativeGainDisplay & $0.00$ \\
 & Matched random & \RZeroHeirCoverage & \RZeroHeirRandomUseful & \RZeroHeirRandomHarmful & \RZeroHeirRandomGainDisplay & \RZeroHeirRandomVsNativeDisplay \\
 & \method & \RZeroHeirCoverage & \RZeroHeirUseful & \RZeroHeirHarmful & \RZeroHeirGainDisplay & \RZeroHeirDeltaDisplay \\
\midrule
\multirow{3}{*}{Agent0}
 & Native commit-all & \AgentHeirNativeCoverage & \AgentHeirNativeUseful & \AgentHeirNativeHarmful & \AgentHeirNativeGainMilli & $0.00$ \\
 & Matched random & \AgentHeirCoverage & \AgentHeirRandomUseful & \AgentHeirRandomHarmful & \AgentHeirRandomGainMilli & \AgentHeirRandomVsNativeMilli \\
 & \method & \AgentHeirCoverage & \AgentHeirUseful & \AgentHeirHarmful & \AgentHeirGainMilli & \AgentHeirDeltaMilli \\
\bottomrule
\end{tabular}

%% file: appendix/extended_self_evolution_results.tex
\section{Extended self-evolution results}
\label{app:extended-results}

\begin{figure}[t]
  \centering
  \includegraphics[width=\textwidth]{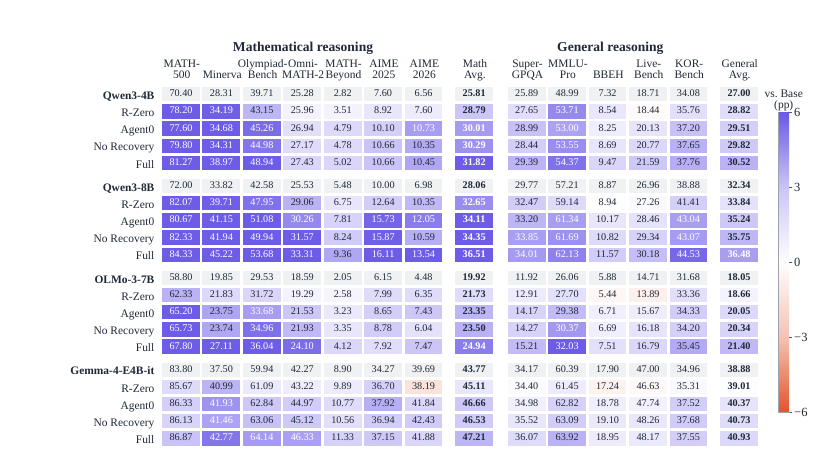}
  \caption{\textbf{Complete self-evolution results (\%).} Every benchmark,
  backbone, and method after ten rounds. Numbers are three-seed means, and
  color shows the change from each backbone's starting checkpoint (Base, gray
  rows).}
  \label{fig:complete-results}
\end{figure}

\paragraph{Benchmarks and scoring.}
The mathematical suite contains
MATH-500~\citep{hendrycks2021measuringmathematicalproblemsolving},
Minerva~\citep{lewkowycz2022solvingquantitativereasoningproblems},
OlympiadBench~\citep{he2024olympiadbenchchallengingbenchmarkpromoting},
Omni-MATH-2~\citep{ballon2026benchmarkssaturatemodelgets},
MATH-Beyond~\citep{mayilvahanan2025mathbeyondbenchmarkrlexpand}, and AIME
2025/2026, with 500, 272, 674, 4,159, 146, 30, and 30 evaluated items after
normalized de-duplication. The first five use one greedy response and
exact-answer accuracy. Each AIME year uses 32 responses per problem at
temperature $0.6$, top-$p=0.95$, and top-$k=20$, and mean response
correctness estimates pass@1. The general suite contains
SuperGPQA~\citep{pteam2025supergpqascalingllmevaluation},
MMLU-Pro~\citep{wang2024mmluprorobustchallengingmultitask},
BBEH~\citep{kazemi2025bigbenchextrahard}, LiveBench
Reasoning~\citep{white2025livebenchchallengingcontaminationlimitedllm}, and
KOR-Bench~\citep{ma2025korbenchbenchmarkinglanguagemodels}, with 26,529,
12,032, 4,520, 200, and 1,250 items, one greedy response each, and
per-question accuracy, except that LiveBench averages its three Reasoning
tasks and KOR-Bench its five categories. Math Avg.\ and General Avg.\ weight
their benchmarks equally, and every evolved score is a three-seed mean after
ten rounds.

%% file: appendix/parameter_experience_inheritance.tex
\paragraph{Donor--recipient interventions.}
\label{app:donor-interventions}
Each intervention fixes the root $S_t$, the recipient $S_t^{\rm sel}$, a
donor $S^{\rm don}$, the qualification rule, and the absorption budget, and
absorbs only that donor's qualified experience,
$S^{\rm abs}=\mathcal{R}(S_t^{\rm sel};D_{\rm don})$.
Figure~\ref{fig:experience-value}a plots the donor gain
$x_{\rm don}=100\,\E[q(S^{\rm don};z)-q(S_t;z)]$ against the absorption gain
$y_{\rm abs}=100\,\E[q(S^{\rm abs};z)-q(S_t^{\rm sel};z)]$, with exact-answer
correctness $q$ on independently evaluated items. The donor axis is measured
against the current root, including for archived donors, and each reference is
attributed to its generating checkpoint.

\paragraph{Objective controls.}
\label{app:objective-controls}
A separate fixed-recipient diagnostic uses 2,400 problems, of which the
recipient initially solves 1,224 (51\%), with three absorption-order seeds that
share the recipient, problems, and references. Repair only sets
$\lambda_{\rm keep}=0$ with the same trajectories and normalization. Recovery
has a net gain of $3.57\pm0.30$ pp, against $1.11\pm0.86$ for Cache SFT and
$2.26\pm0.45$ for Repair only, because it loses far fewer existing successes,
as Figure~\ref{fig:recovery-controls}a shows. A larger preservation weight
lowers both acquisition and forgetting in
Figure~\ref{fig:recovery-controls}b, and $\lambda_{\rm keep}=1$ gives the
largest net gain.

\paragraph{Matched source replacement.}
\label{app:history-resource}
Three seed groups with four late-round recipients each compare 32 current
records plus 32 qualified records from non-inherited sources with the same
current records plus same-age records from inherited sources, matching the
record count and the repair and keep group sizes. Age zero uses identical
material. Pointwise 95\% intervals in Figure~\ref{fig:experience-value}d use
2,000 hierarchical paired bootstrap draws that resample seed groups and then
recipients. Figure~\ref{fig:recovery-controls}c shows the two conditions'
absorption gains, and Figure~\ref{fig:recovery-controls}d shows that Recovery
beats Cache SFT at every recipient checkpoint of
Figure~\ref{fig:experience-value}.

\begin{figure}[!htbp]
  \centering
  \includegraphics[width=\textwidth]{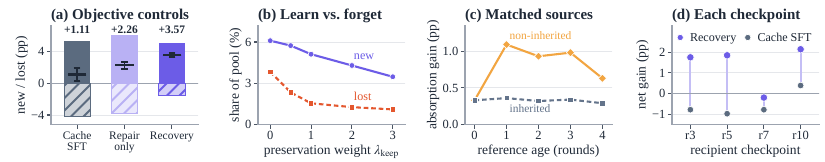}
  \caption{\textbf{Supplementary interventions.} (a) Objective controls on a
  fixed recipient and 2,400 problems, with mean $\pm$ SD net gain over three
  absorption-order seeds. (b) New and lost successes as the preservation weight
  changes. (c) Absorption gain with non-inherited or same-age inherited history
  at 64 records. (d) Net gain of Recovery and Cache SFT at each recipient
  checkpoint.}
  \label{fig:recovery-controls}
  \label{fig:source-replacement}
\end{figure}

%% file: appendix/compute_performance.tex
\section{Compute and budget accounting}
\label{app:compute-performance}
\label{app:compute-accounting}
\label{app:compute-sensitivity}

\paragraph{Hardware and elapsed time.}
Training used two NVIDIA RTX PRO 6000 Blackwell Server Edition GPUs with about
96~GB of memory each. A Full run took about 8--9 wall-clock hours per round, or
80--90 hours for ten rounds, and benchmark evaluation took about 4--5 hours in
total. The training-side cost includes task generation, native branch
training, recipient attempts, verification, teacher and student passes,
screening, and confirmation, also in rounds that retain the root. The
fixed-round results in Table~\ref{tab:self-evolution} use the same horizon and
three-seed aggregation across methods.

\paragraph{Accounting budget.}
Figure~\ref{fig:budget-sensitivity} allocates native training under a
normalized accounting budget. A single-stream native baseline uses 100.0 units,
and No Recovery and Full use 111.9 and 123.7 units in their default ten-round
configurations. Generation, Scout optimization, and native learner optimization
account for 38.0, 6.0, and 51.0 units in each configuration. Full adds
screening (4.8), recipient attempts (3.2), verification (2.8), teacher passes
(2.4), absorption (5.2), confirmation (6.4), and state export (3.9). Budget 1.0
equals Full's 123.7 units. Every variant keeps ten rounds and the same
generation and Scout schedule, with fixed costs and branch counts of
$(49.0,1)$ for the native baseline, $(64.1,4)$ for the more-candidates control,
$(60.9,2)$ for No Recovery, and $(72.7,2)$ for Full. With cap $C$ and $K$
native branches, each branch receives
$n_{\rm native}=\lfloor (C-C_{\rm fixed})/(10K\cdot0.17)\rfloor$ calls per
round, set before observing results. The sensitivity sweeps keep
$m_{\rm scr}=16$ and the Recovery caps fixed, use $m_{\rm con}=32$ for the
candidate sweep and $K=2$ for the confirmation sweep, and charge a fixed cost
of $63.1+1.6K+0.2m_{\rm con}$ units. Accounting units are not measured
GPU-hours or FLOPs.

%% file: appendix/prompt_templates.tex
\section{Prompt templates and input formats}
\label{app:prompt-templates}
\newtcolorbox{promptpanel}[1]{%
  arc=1pt, boxrule=.4pt, colframe=ink!40, colback=white,
  colbacktitle=ink!8, coltitle=ink!75, fonttitle=\small,
  left=6pt, right=6pt, top=3pt, bottom=3pt,
  before skip=5pt, after skip=5pt,
  fontupper=\ttfamily\scriptsize\raggedright
    \setlength{\parskip}{.45\baselineskip},
  title={#1}}

The Curriculum Scout proposes tasks, and the Evolving Learner trains,
attempts problems, and receives Recovery updates, as in
Appendix~\ref{app:native-hyperparameters}.

\subsection{Native problem generation}
\label{app:prompt-native}

The Curriculum Scout uses the same two messages in every round, for native
training and for candidate-pool generation. Its proposed boxed answer is not a
training label: native admission assigns the consensus label of the Evolving
Learner's responses.

\begin{promptpanel}{Curriculum Scout: system message}
Construct an original, challenging competition-mathematics problem.

Work out the problem and its answer privately before writing your response.
You may draw on algebra, geometry, number theory, combinatorics, prealgebra,
probability, statistics, calculus, or another mathematical area. Choose a
challenge that fewer than 30\% of advanced high-school students would solve.
Familiar textbook exercises and recognizable contest questions should not be
reused.

Your response must contain two parts in this order: the problem statement
enclosed by \textless question\textgreater{} and
\textless /question\textgreater{}, then its final answer enclosed by
\verb|\boxed{}|. Do not include your planning, a worked solution, or any
other commentary. Follow this layout:

\textless question\textgreater\\
\{problem\_statement\}\\
\textless /question\textgreater

\verb|\boxed{final_answer}|
\end{promptpanel}

\begin{promptpanel}{Curriculum Scout: user message}
Write one problem now. Use the required statement-and-answer layout, with no
additional text.
\end{promptpanel}

\subsection{Evolving Learner problem-solving input}
\label{app:prompt-evolution}

Native rollouts, admission responses, recipient attempts, and the screening
and confirmation pools use the same instruction, with the problem as the user
message.

\begin{promptpanel}{Evolving Learner: problem-solving system message}
Develop a solution with explicit intermediate reasoning. Conclude with the
resulting answer in \verb|\boxed{}|.
\end{promptpanel}

\begin{promptpanel}{Evolving Learner: problem-solving user message}
\{problem\_text\}
\end{promptpanel}

\subsection{Reference-conditioned repair input}
\label{app:prompt-repair}

The student and the keep teacher see only the input above. The repair teacher
also receives one qualified reference, which may come from an unselected
branch or an earlier round, and runs forward passes on the recipient's fixed
recorded response, aligned token by token with the student.

\begin{promptpanel}{Repair-teacher context}
system: Develop a solution with explicit intermediate reasoning. Conclude
with the resulting answer in \verb|\boxed{}|.\\
user: A successful training example is provided below.\\
\{qualified\_reference\_response\}

Use this example as guidance when solving the following problem:\\
\{problem\_text\}\\
assistant:
\end{promptpanel}

\subsection{Answer verification and confirmation}
\label{app:prompt-verification}

Recovery eligibility uses training-role records, response completion, answer
verification, and reference provenance, and excludes known contradictory
problems, wrong labels, and erroneous references. Screening and confirmation
are numerical decisions on fresh responses: the runs contain no
natural-language judge prompt and no model-generated accept or reject
decision. External benchmarks use their frozen standard instructions, and
their prompts, responses, and results never enter training, memory, screening,
or confirmation.
\label{app:prompt-external}